\documentclass[sigconf]{acmart}

\AtBeginDocument{%
  }

\setcopyright{acmlicensed}
\copyrightyear{2026}
\acmYear{2026}
\setcopyright{cc}
\setcctype{by}
\acmConference[KDD 2026] {Proceedings of the 32nd ACM SIGKDD Conference on Knowledge Discovery and Data Mining V.2}{August 9--13, 2026}{Jeju Island, Republic of Korea.}
\acmBooktitle{Proceedings of the 32nd ACM SIGKDD Conference on Knowledge Discovery and Data Mining V.2 (KDD 2026), August 9--13, 2026, Jeju Island, Republic of Korea}
\acmISBN{979-8-4007-2259-2/2026/08}
\acmDOI{10.1145/3770855.3818864}

\usepackage{graphicx}
\usepackage[table]{xcolor}
\usepackage{mathtools}
\usepackage{amsthm}
\usepackage{bbm}
\usepackage{tasks}
\usepackage{enumitem}
\graphicspath{{img/}}
\usepackage{wrapfig}
\usepackage[capitalise]{cleveref}
\usepackage{float}
\usepackage{booktabs}
\usepackage{url}
\usepackage{pifont}
\usepackage{makecell}
\usepackage{multicol}
\usepackage{tikz-cd}
\usepackage{units}
\usepackage[ruled]{algorithm2e}
\usepackage{subcaption}
\usepackage{titletoc}
\usepackage{appendix}
\usepackage{dsfont}
\usepackage{multirow}

\providecommand{\Log}{\operatorname{Log}}
\providecommand{\Exp}{\operatorname{Exp}}

\providecommand{\Diag}{\operatorname{Diag}}

\newcommand{\cor}[1]{\mathcal{C}_{++}^{#1}}

\newcommand{\off}{\operatorname{Off}}
\newcommand{\offlog}{\operatorname{\Log^{o}}}
\newcommand{\offexp}{\operatorname{\Exp^{o}}}
\newcommand{\lslog}{\operatorname{\Log^{\star}}}
\newcommand{\lsexp}{\operatorname{\Exp^{\star}}}

\theoremstyle{plain}
\newtheorem{theorem}{Theorem}[section]
\newtheorem{proposition}[theorem]{Proposition}
\newtheorem{lemma}[theorem]{Lemma}
\newtheorem{corollary}[theorem]{Corollary}
\theoremstyle{definition}
\newtheorem{definition}[theorem]{Definition}

\theoremstyle{remark}

\providecommand{\dist}{\operatorname{d}}

\providecommand{\Cov}{{\mathrm{Cov}}}

\providecommand{\Log}{\operatorname{Log}}
\providecommand{\Exp}{\operatorname{Exp}}

\providecommand{\argmin}{\operatorname{argmin}}

\providecommand{\mlog}{\operatorname{mlog}}
\providecommand{\mexp}{\operatorname{mexp}}

\providecommand{\mexp}{\operatorname{mexp}}

\crefname{equation}{Eq.}{Eqs.}
\Crefname{equation}{Equation}{Equations}
\crefname{figure}{Fig.}{Figs.}
\Crefname{figure}{Figure}{Figures}
\crefname{table}{Tab.}{Tabs.}
\Crefname{table}{Table}{Tables}
\crefname{algocf}{Alg.}{Algs.}
\Crefname{algocf}{Algorithm}{Algorithms}
\crefname{section}{Sec.}{Secs.}
\Crefname{section}{Section}{Sections}
\crefname{appendix}{App.}{Apps.}
\Crefname{appendix}{Appendix}{Appendices}

\crefname{theorem}{Thm.}{Thms.}
\Crefname{theorem}{Theorem}{Theorems}
\crefname{lemma}{Lem.}{Lems.}
\Crefname{lemma}{Lemma}{Lemmas}
\crefname{definition}{Def.}{Defs.}
\Crefname{definition}{Definition}{Definitions}
\crefname{corollary}{Cor.}{Cors.}
\Crefname{corollary}{Corollary}{Corollaries}
\crefname{remark}{Rem.}{Rems.}
\Crefname{remark}{Remark}{Remarks}
\crefname{proposition}{Prop.}{Props.}
\Crefname{proposition}{Proposition}{Propositions}
\crefname{proof}{Pr.}{Prs.}
\Crefname{proof}{Proof}{Proofs}
\definecolor{gain}{RGB}{0,191,65}
\usepackage{tabularx}
\newcolumntype{Y}{>{\centering\arraybackslash}X}

\begin{document}

\title{A Sliced-Wasserstein Framework on Correlation Matrices for EEG Decoding}

\author{Chen Hu}
\authornote{Equal contribution.}
\email{huchen.ml@gmail.com}
\affiliation{%
  \institution{Westlake University}
  \city{Hangzhou}
  \country{China}
}

\author{Rui Wang}
\authornotemark[1]
\email{cs_wr@jiangnan.edu.cn}
\affiliation{%
  \institution{School of Artificial Intelligence and Computer Science}
  \institution{Jiangnan University}
  \city{Wuxi}
  \country{China}
}

\author{Jiale Zhou}
\email{zhoujiale@westlake.edu.cn}
\affiliation{
  \institution{Westlake University}
  \city{Hangzhou}
  \country{China}
}
\affiliation{
  \institution{Zhejiang University}
  \city{Hangzhou}
  \country{China}
}

\author{Jingjun Yi}
\email{rsjingjuny@whu.edu.cn}
\affiliation{
  \institution{Westlake University}
  \city{Hangzhou}
  \country{China}
}

\author{Shaocheng Jin}
\email{shaochengjin.ai@gmail.com}
\affiliation{%
  \institution{School of Artificial Intelligence and Computer Science}
  \institution{Jiangnan University}
  \city{Wuxi}
  \country{China}
  }

\author{Yidong Song}
\authornote{Corresponding authors.}
\email{songyd6@mail2.sysu.edu.cn}
\affiliation{%
 \institution{Sun Yat-sen University}
 \city{Guangzhou}
 \country{China}}

\author{Yefeng Zheng}
\authornotemark[2]
\email{zhengyefeng@westlake.edu.cn}
\affiliation{%
  \institution{Westlake University}
  \city{Hangzhou}
  \country{China}
}

\renewcommand{\shortauthors}{Chen Hu et al.}

\begin{abstract}
Electroencephalography (EEG) offers noninvasive, millisecond resolution recordings of neuronal activity and is widely used in neuroscience and healthcare. 
Many EEG decoding pipelines rely on covariance descriptors for their robustness to noise, but such representations are sensitive to channel-wise scaling.
Recent studies have therefore advocated full-rank correlation matrices as a scale-invariant alternative for EEG decoding.
In this paper, we study Sliced-Wasserstein (SW) discrepancies for probability distributions on the manifold of full-rank correlation matrices. We adopt the pullback-Euclidean formulation of SW, referred to as Pullback Euclidean Metric Sliced-Wasserstein (PEMSW), and instantiate it under two recently introduced correlation geometries, \textit{i.e.}, the Off-Log Metric (OLM) and Log-Scaled Metric (LSM). This yields two Correlation Sliced-Wasserstein (CorSW) discrepancies with closed-form slicing coordinates and efficient computation through one-dimensional Wasserstein distances. Building on CorSW, we further develop a domain generalization (DG) framework for EEG decoding.
Experiments on three EEG datasets demonstrate improved generalization under distribution shifts, with low training overhead and no additional inference cost.
The source code is available at \href{https://github.com/ChenHu-ML/CorSW}{github.com/ChenHu-ML/CorSW}.
\end{abstract}







\maketitle

\section{Introduction}
\label{sec:intro}
Electroencephalography (EEG) offers a unique, non-invasive window into brain dynamics by recording neural activity at millisecond time scales. 
This high temporal resolution has supported a wide spectrum of neuroscience and clinical applications, ranging from seizure detection~\citep{Ahmad2022SeizureDetection, Cherian2022SeizureDetection} and sleep staging~\citep{Aboalayon2016SleepClassification, Phan2022SleepStaging} to motor imagery decoding~\citep{Altaheri2023MotorImagery, ju2023graph, Roy2019ChronoNet}, large-scale abnormality screening~\citep{Roy2019ChronoNet}, affective state analysis~\citep{Suhaimi2020EmotionRecognition, Dadebayev2022EmotionRecognition}, and attention tracking in auditory~\citep{Biesmans2016AuditoryAttention}. 
However, EEG signals are inherently noisy and non-stationary, with statistics that vary substantially across subjects and sessions~\citep{Hine2017RestingStateEEG}, resulting in pronounced distribution shifts and poor generalization.

A common representation in EEG decoding summarizes the power and spatial distribution of a multi-channel EEG window using covariance matrices~\citep{blankertz2007optimizing}. 
Such matrices are typically symmetric positive definite (SPD) and naturally lie on Riemannian manifolds rather than Euclidean spaces~\citep{bhatia2009positive}. 
This has fueled growing interest in deep geometric learning, which extends neural networks to non-Euclidean domains by respecting manifold structure.
Recent models operate directly on SPD manifolds to leverage their
Riemannian geometry~\citep{spdnet, grnet, spdnetbn, manifoldnet, gdlnet, wang2025learning, chen2023riemannian, chen2024adaptive,chen2024liebn,chen2024rmlr,chen2024rmlrspd,chen2025gyrogroup,chen2025understanding}, and SPD-based methods have become competitive baselines for EEG decoding~\citep{kobler2022spd, matt, ju2023graph, ju2024deep}.

Nevertheless, covariance-based representations are not always ideal in scientific settings where absolute scales act as nuisance factors.
This motivates the use of correlation matrices as normalized and scale-invariant alternatives~\citep{david2019riemannian,thanwerdas2024permutation}.
Correlation matrices are widely used when relative dependencies matter more than magnitudes, including financial econometrics~\citep{rebonato2011most}, brain-computer interfaces (BCI)~\citep{jalili2011constructing, hu2025correlation}, and Gaussian graphical modeling~\citep{epskamp2018tutorial}.
In EEG analysis, where electrode signal amplitudes often vary substantially, correlation matrices emphasize inter-channel dependencies and remain informative.
Despite their importance, correlation matrices have received limited attention in deep geometric learning, as they form a quotient manifold of SPD matrices under the congruence action of positive diagonal matrices~\citep{david2019riemannian,thanwerdas2024permutation} and do not admit standard SPD geometry.
Only recently have intrinsic Riemannian structures been developed for the correlation manifold, notably the Off-Log Metric (OLM) and Log-Scaled Metric (LSM)~\citep{thanwerdas2024permutation}.

Sets of covariance or correlation matrices extracted from EEG recordings can be naturally viewed as samples from underlying probability distributions~\citep{yair2019domain}, motivating the use of distributional comparison tools.
This naturally calls for the development of computationally efficient tools that operate directly at the distribution level. 
In Euclidean spaces, optimal transport (OT) provides a principled framework for comparing probability distributions and induces the Wasserstein distance~\citep{villani2009optimal, peyre2019computational}.
OT can be extended to Riemannian manifolds in a way that respects the underlying  geometry~\citep{mccann2001polar, cui2019spherical, ju2022deep}.
However, computing Wasserstein distances typically requires solving large-scale linear programs over the samples, with a computational complexity of $\mathcal{O}\left(n^3\log(n)\right)$ in the number of samples $n$, which becomes prohibitive for large-scale learning. 

To alleviate this burden, several alternatives have been proposed, such as entropic regularization~\citep{cuturi2013sinkhorn} or mini-batch approximations~\citep{fatras2019learning}. 
Among them, the Sliced-Wasserstein  (SW) distance is particularly appealing, which computes the average Wasserstein distance over one-dimensional projections~\citep{rabin2011wasserstein}, substantially reducing computational cost while preserving key topological and statistical properties~\citep{bonnotte2013unidimensional, nadjahi2020statistical, bayraktar2021strong}. 
Building on the success of the Euclidean SW distance, there has been growing interest in generalizing SW to non-Euclidean spaces. Notably, recent works have generalized SW to several classes of vector manifolds, including spherical~\citep{bonet2022spherical} and hyperbolic spaces~\citep{bonet2022hyperbolic}. More recently, SW has also been extended to the manifold of SPD matrices~\citep{bonet2023sliced}.
Nevertheless, a principled SW construction tailored to the correlation manifold remains missing.

A useful route for extending SW beyond Euclidean spaces is to exploit geometries induced by global embeddings. In particular, \citet{bonet2025sliced} developed Sliced-Wasserstein discrepancies on Cartan-Hadamard manifolds, including the pullback-Euclidean setting and applications to SPD manifolds. This perspective is especially convenient when the manifold geometry is induced by a diffeomorphic embedding, since geodesic slicing can be implemented as linear slicing in the embedding space. Motivated by this viewpoint, we study SW discrepancies on the manifold of full-rank correlation matrices. Under the OLM and LSM geometries~\citep{thanwerdas2024permutation}, correlation matrices admit explicit pullback-Euclidean representations, allowing us to define practical Correlation Sliced-Wasserstein (CorSW) discrepancies with closed-form slicing coordinates.
We finally demonstrate the practical relevance of CorSW by using it to build a domain generalization framework for EEG decoding.
In summary, our main contributions are as follows:
\begin{itemize}
    \item We introduce Correlation Sliced-Wasserstein (CorSW) discrepancies for probability distributions on the manifold of full-rank correlation matrices. Under the OLM and LSM geometries, CorSW admits explicit slicing coordinates and efficient Monte-Carlo estimation via one-dimensional Wasserstein distances.

    \item We incorporate CorSW into a source-only domain generalization framework for EEG decoding, aligning correlation-valued representations across source domains without requiring target-domain data.

    \item We evaluate the proposed framework on three benchmark EEG datasets and show consistent improvements under distribution shifts, with low training overhead and no additional inference-time cost.
\end{itemize}
\section{Preliminary}
\label{sec:preliminary}
\begin{table*}[t]
    \centering
    \caption{Summary of diffeomorphic mappings, geodesic distances, and geodesics under OLM and LSM~\cite{thanwerdas2024permutation}.}
    \label{tab:cor_manifold_ops}
    \resizebox{0.97\linewidth}{!}{
    \begin{tabular}{c|ccc}
        \toprule
        \textbf{Metric} & \textbf{Diffeomorphisms} & \textbf{Geodesic Distance $\dist(C_1, C_2)$ } & \textbf{Geodesic from $C_1$  to $C_2, t\in[0,1]$} \\
        \midrule
        \textbf{OLM} 
        & 
        $ \begin{aligned}
        &\offlog: C \in \cor{n} \mapsto \off \left(\mlog(C)\right) \in \mathrm{Hol}(n) \\
        &\offexp: S \in \mathrm{Hol}(n) \mapsto \mexp \left(S + \mathcal{D}^{o}(S)\right) \in \cor{n}  \\
        \end{aligned} \quad\quad$ 
        & 
        $\dist^{o}(C_1, C_2) = \|\offlog(C_1) - \offlog(C_2)\|_F$  
        & 
        $\offexp \bigl((1-t)\offlog(C_1) + t\,\offlog(C_2)\bigr)$ 
        \\
        \midrule

        \textbf{LSM} 
        & 
        $\begin{aligned}
        &\lslog: C \in \cor{n} \mapsto \mlog \left(\mathcal{D}^\star(C)\,C\,\mathcal{D}^\star(C)\right) \in \mathrm{Row}_0(n)\\
        &\lsexp: S \in \mathrm{Row}_0(n) \mapsto \operatorname{Cor} \left(\mexp(S)\right) \in \cor{n}\\
        \end{aligned}$ 
        & 
        $\dist^{\star}(C_1, C_2) = \|\lslog(C_1) - \lslog(C_2)\|_F$  
        & 
        $\lsexp \bigl((1-t)\lslog(C_1) + t\,\lslog(C_2)\bigr)$ 
        \\
        \bottomrule
    \end{tabular}
    }
\end{table*}

In this section, we briefly review the PEM, the geometry of full-rank correlation matrices, and the Wasserstein distance.
For more in-depth discussions, please refer to
\citep{do1992riemannian,villani2009optimal,loring2011introduction,david2019riemannian,thanwerdas2024permutation}.

\paragraph{\textbf{Pullback Euclidean metrics.}}
Let $\mathcal{M}$ be a smooth manifold and $(\mathcal{E},\langle\cdot,\cdot\rangle_{\mathcal{E}})$ denote a finite-dimensional inner-product space. Given a global diffeomorphism $\phi:\mathcal{M}\to\mathcal{E}$, the PEM on $\mathcal{M}$ is defined as
\begin{equation}
(\phi^{\ast} g_{\mathcal{E}})_x(\xi,\eta)
\triangleq \bigl\langle \phi_{*,x}(\xi),\,\phi_{*,x}(\eta)\bigr\rangle_{\mathcal{E}}, \xi,\eta\in T_x\mathcal{M},
\end{equation}
where $\phi_{*,x}:T_x\mathcal{M}\to T_{\phi(x)}\mathcal{E}\simeq\mathcal{E}$ is the differential at $x$. In this case, for any $x_0,x_1\in\mathcal{M}$, the geodesic $\gamma(x_0, x_1, t)$ and the induced distance $d_{\mathcal{M}}(x_0,x_1)$ admit closed-form expressions, given by
\begin{align}
\label{eq:pullback_geodesic}
\gamma(t) &= \phi^{-1} \bigl((1-t)\phi(x_0)+t\,\phi(x_1)\bigr), \\
\label{eq:pullback_metric}
d_{\mathcal{M}}(x_0,x_1) &= \|\phi(x_0)-\phi(x_1)\|_{\mathcal{E}} .
\end{align}

\paragraph{\textbf{Correlation matrices}.}
Given random vector $X$ with invertible covariance matrix $P = (\Cov(X_i,X_j))_{1\le i,j\le n}$, its correlation matrix is defined by variance normalization as
\begin{equation}
\label{eq:cov2corr}
C = \operatorname{Cor}(P) = \Diag(P)^{-\frac{1}{2}}\, P\, \Diag(P)^{-\frac{1}{2}} \in \cor{n},
\end{equation}
where $\Diag(P)$ denotes the diagonal matrix of $P$, and $\cor{n}$ is the set of full-rank correlation matrices.
Recent work has shown that $\cor{n}$ admits several Riemannian metrics induced by diffeomorphisms, which yield closed-form expressions for geodesic distances and curves~\citep{thanwerdas2024permutation}. In this work, we focus on two metrics: the OLM and the LSM, as summarized in \cref{tab:cor_manifold_ops}.
For OLM, the diffeomorphic pair $(\offlog,\offexp)$ maps $\cor{n}$ to $\mathrm{Hol}(n)=\{X\in\mathbb{R}^{n\times n}\mid X=X^\top,\ \Diag(X)=\mathbf{0}\}$, where $\off(\cdot)$ extracts the off-diagonal part. Following \citet[Sec.~3.3]{archakov2021new}, the diagonal correction $\mathcal{D}^{o}(S)$ in $\offexp$ is the unique diagonal matrix such that \footnote{In implementation, $\mathcal{D}^{o}(S)$ is obtained by fixed-point iteration.}
\begin{equation}
\log \left(\Diag\bigl(\exp(S+\mathcal{D}^{o}(S))\bigr)\right)=\mathbf{0}.
\end{equation}

For LSM, $(\lslog,\lsexp)$ maps $\cor{n}$ to $\mathrm{Row}_0(n)=\{X\in\mathbb{R}^{n\times n}\mid X=X^\top,\ X\mathds{1}=\mathbf{0}\}$. The scaling matrix $\mathcal{D}^\star(C)$ is defined as the unique positive solution of
\begin{equation}
\label{eq:lsm_log_newton}
f:x\in\mathbb{R}^n \mapsto Cx - \frac{1}{x},
\end{equation}
which can be solved using a damped Newton method \citep[Sec.~3.5]{thanwerdas2024permutation}.
In all experiments, we use the standard fixed-point solver for the OLM diagonal correction and damped Newton iterations for the LSM scaling, initialized at the zero diagonal correction and the all-one scaling vector, respectively.
These mappings are evaluated once per batch before the one-dimensional SW computations; CorSW therefore requires no iterative OT solver.
Consequently, under both OLM and LSM, geodesic distances and geodesics on $\cor{n}$ admit closed forms
as listed in \cref{tab:cor_manifold_ops}.

\paragraph{\textbf{Wasserstein distance}.}
Let $\mathcal{P}_p(\mathbb{R}^d)$ be the set of probability measures on $\mathbb{R}^d$ with finite $p$-th moment for $p \ge 1$.
In general, the $p$-Wasserstein distance between $\mu,\nu \in \mathcal{P}_p(\mathbb{R}^d)$ is defined as
\begin{equation}
\label{eq:wasserstein}
W_p^p(\mu,\nu) = \inf_{\gamma \in \Pi(\mu,\nu)} \int_{\mathbb{R}^d \times \mathbb{R}^d} \|x-y\|_2^p \, \mathrm{d}\gamma(x,y),
\end{equation}
where $\Pi(\mu,\nu)$ is the set of joint distributions with marginals $\mu$ and $\nu$.
For empirical measures based on $n$ samples, computing $W_p$ requires solving a linear program, whose complexity is $O(n^3\log n)$ \citep{pele2009fast}, making it impractical for large-scale learning problems. 

A key simplification occurs in the one-dimensional case. When $d=1$, $W_p$ admits an explicit formula \citep[Remark~2.30]{peyre2019computational}:
\begin{equation}
\label{eq:1d-wasserstein}
W_p^p(\mu,\nu) = \int_0^1 \big|F_\mu^{-1}(u) - F_\nu^{-1}(u)\big|^p \, \mathrm{d}u,
\end{equation}
where $F_\mu^{-1}$ and $F_\nu^{-1}$ denote the quantile functions of $\mu$ and $\nu$ respectively, which can be computed from $n$ samples in $O(n\log n)$ time.
Building on this tractability, the SW distance provides a scalable approximation of $W_p$ in higher dimensions~\citep{rabin2011wasserstein,bonneel2015sliced}.
It is defined by averaging Wasserstein distances between projected measures.
\begin{equation}
\mathrm{SW}_p^p(\mu,\nu)
=\int_{S^{d-1}} W_p^p \left(t^\theta_\#\mu,\ t^\theta_\#\nu\right)\,\mathrm{d}\lambda(\theta),
\end{equation}
where $\lambda$ is the uniform measure on $S^{d-1}$. In practice, the integral is approximated via Monte Carlo sampling using $L$ projection directions, resulting in a computational complexity of $O\big(L n (d + \log n)\big)$.
Beyond computational efficiency, $\mathrm{SW}_p$ enjoys strong theoretical guarantees, including dimension-independent sample complexity and a Hilbertian structure \citep{kolouri2016sliced,nadjahi2019asymptotic,nadjahi2020statistical}.
\section{Correlation SW via Pullback Slicing}
\label{sec:pullback_sw}

This section develops CorSW from a pullback slicing perspective. We first introduce the notation for SW discrepancies on manifolds whose geometry is induced by a diffeomorphic embedding, in the pullback-Euclidean setting. We then specialize this construction to correlation manifolds $\cor{n}$ under OLM and LSM, yielding explicit CorSW formulations. The proofs are deferred to \cref{sec:proofs}.

\subsection{Pullback Euclidean SW}
\label{sec:pem_sw}

In Euclidean spaces, the SW distance is obtained by projecting measures onto one-dimensional lines and averaging the resulting $1$D Wasserstein distances. We begin with the pullback-Euclidean setting, where a global diffeomorphism $\phi$ maps the manifold $\mathcal{M}$ to an embedding space $\mathcal{E}$. Under this structure, geodesics on $\mathcal{M}$ become straight lines in $\mathcal{E}$, so geodesic slicing on the manifold can be implemented as standard linear slicing in the embedding space.

Let $(\mathcal{M},\phi^{\ast}g_{\mathcal{E}})$ be a pullback manifold as in \cref{eq:pullback_metric}.
We fix a reference point $x_0\in\mathcal{M}$ and write $\langle\cdot,\cdot\rangle_{\mathcal{E}}$ for the Euclidean inner product on $\mathcal{E}$.
the unit sphere in $\mathcal{E}$ can be defined as
\begin{equation}
S_{\mathcal{E}}
=
\{a\in\mathcal{E}:\|a\|_{\mathcal{E}}=1\},
\end{equation}
equipped with the uniform measure $\lambda_{\mathcal{E}}$.

\paragraph{\textbf{Geodesic slicing}.}
We consider geodesics through the reference point $x_0$ and use the corresponding signed scalar coordinate as a one-dimensional projection.
For a direction $a\in S_{\mathcal{E}}$, the pullback Euclidean structure gives the geodesic line
\begin{equation}
\label{eq:pem_geodesic_line}
\mathcal{G}_a
=\Bigl\{\phi^{-1}\bigl(\phi(x_0)+t\,a\bigr):t\in\mathbb{R}\Bigr\}.
\end{equation}
The slicing coordinate of $x\in\mathcal{M}$ is obtained by $\phi(x)$ onto the direction $a \in \mathcal{E}$, which yields the closed-form expressions in \cref{thm:pem_projection}.

\begin{theorem}[Projection and slicing coordinate on PEMs]
\label{thm:pem_projection}
Let $a\in S_{\mathcal{E}}$. For any $x\in\mathcal{M}$, the projection of $x$ onto $\mathcal{G}_a$ admits the form
\begin{equation}
\label{eq:pem_projection}
P^{\mathcal{G}_a}(x)
=\phi^{-1}\!\Bigl(\phi(x_0)+t_a(x)\,a\Bigr)\in\mathcal{G}_a,
\end{equation}
where the associated slicing coordinate is
\begin{equation}
\label{eq:pem_coordinate}
t_a(x)=\bigl\langle a,\ \phi(x)-\phi(x_0)\bigr\rangle_{\mathcal{E}}.
\end{equation}
\end{theorem}

\paragraph{\textbf{Pullback SW discrepancy}.}
We now define the SW discrepancy induced by pullback slicing. 
For $p\ge 1$, define the $p$-moment class on $(\mathcal{M},d_{\mathcal{M}})$ by
\begin{equation}
\mathcal{P}_p(\mathcal{M})
=\Bigl\{\mu\in\mathcal{P}(\mathcal{M}):\int_{\mathcal{M}} d_{\mathcal{M}}(x,x_0)^p,\mathrm{d}\mu(x)<\infty\Bigr\}.
\end{equation}

We then average one-dimensional Wasserstein distances between the push-forward measures of slicing coordinates \cref{eq:pem_coordinate} over random directions $a\in S_{\mathcal{E}}$.

\begin{definition}[Pullback Euclidean SW]
\label{def:pem_sw}
Let $p\ge 1$ and $\mu,\nu\in\mathcal{P}_p(\mathcal{M})$. The pullback Euclidean SW discrepancy is defined by
\begin{equation}
\label{eq:pem_sw_def}
\mathrm{PEMSW}_p(\mu,\nu)^p = \int_{S_{\mathcal{E}}} W_p^p \Bigl(t_a{}_{\#}\mu,\ t_a{}_{\#}\nu\Bigr), \mathrm{d}\lambda_{\mathcal{E}}(a).
\end{equation}
\end{definition}

The pullback structure $g=\phi^{\ast}_{g_{\mathcal{E}}}$ allows us to evaluate this discrepancy directly in the embedding space.
The slicing coordinate $t_a$ is linear in $\phi(x)$ up to a constant shift, as formalized below.


\begin{lemma}[Translation form of $t_a$]
\label[lemma]{lem:translation_projection}
For $a\in S_{\mathcal{E}}$, define $\pi_a:\mathcal{E}\to\mathbb{R}$ by $\pi_a(z)=\langle a,z\rangle_{\mathcal{E}}$.
Then for all $x\in\mathcal{M}$,
\begin{equation}
\label{eq:ta_shifted}
t_a(x)=\pi_a(\phi(x))-\pi_a(\phi(x_0)).
\end{equation}
Consequently, $t_a{}_\#\mu$ is a translation of $\pi_a{}_\#\tilde{\mu}$, and the same holds for $\nu$.
\end{lemma}

Using the translation invariance of the one-dimensional Wasserstein distance, we obtain the following equivalence in \cref{thm:pem_sw_euclidean}.

\begin{theorem}[Equivalence to Euclidean SW~\citep{bonet2025sliced}]
\label{thm:pem_sw_euclidean}
For any $p\ge 1$ and $\mu,\nu\in\mathcal{P}_p(\mathcal{M})$, $\mathrm{PEMSW}_p$ and $W_p^{\mathcal{M}}$ can be evaluated in the embedding space $\mathcal{E}$ via $\tilde{\mu}=\phi_\#\mu$ and $\tilde{\nu}=\phi_\#\nu$:
\begin{align}
\mathrm{PEMSW}_p(\mu,\nu)=\mathrm{SW}_p(\tilde{\mu},\tilde{\nu}), \;\;\; W_p^{\mathcal{M}}(\mu,\nu)=W_p(\tilde{\mu},\tilde{\nu}).
\end{align}
\end{theorem}

\cref{thm:pem_sw_euclidean} shows that $\mathrm{PEMSW}_p$ is exactly the Euclidean SW distance between the embedded measures.
This equivalence is the main computational simplification of the pullback construction: once manifold-valued samples are mapped by $\phi$, all slicing operations reduce to standard Euclidean projections.
As a result, the same sorting-based implementation used in Euclidean SW can be directly applied in the embedding space.

The slicing coordinate $t_a$ is defined relative to a reference point $x_0$.
We next show that this choice does not affect the value of $\mathrm{PEMSW}_p$.

\begin{proposition}[Reference-point invariance]
\label[proposition]{prop:pem_x0_invariance}
Let $p\ge 1$ and $x_0,x_1\in\mathcal{M}$.
For $a\in S_{\mathcal{E}}$, the slicing coordinates associated with $x_0$ and $x_1$ are defined as
\begin{equation}
t^{(x_i)}_a(x)=\langle a,\phi(x)-\phi(x_i)\rangle_{\mathcal{E}}, \quad i\in\{0,1\}.
\end{equation}
Then for all $\mu,\nu\in\mathcal{P}_p(\mathcal{M})$,
\begin{equation}
W_p\Bigl(t^{(x_0)}_a{}_\#\mu,\ t^{(x_0)}_a{}_\#\nu\Bigr)
=
W_p\Bigl(t^{(x_1)}_a{}_\#\mu,\ t^{(x_1)}_a{}_\#\nu\Bigr),
\quad \forall\, a\in S_{\mathcal{E}}.
\end{equation}
Hence, $\mathrm{PEMSW}_p$ is invariant to the choice of the reference point.
\end{proposition}

\paragraph{\textbf{Theoretical properties}.} 
The equivalence in \cref{thm:pem_sw_euclidean} connects pullback slicing to the standard Euclidean SW distance. 
Thus, $\mathrm{PEMSW}_p$ inherits the metric, topological, comparison, and approximation properties of Euclidean SW through the embedding map $\phi$, as covered by the pullback-Euclidean case of \citet{bonet2025sliced}. 
Below, we focus on two consequences that are useful for interpreting the behavior of pullback slicing in the embedding space. 

First, pullback slicing is sensitive to first-order shifts of the embedded distributions. For each slicing direction, the one-dimensional Wasserstein distance controls the difference between the projected means; averaging over directions gives the following bound.

\begin{proposition}[Moment lower bound]
\label[proposition]{prop:pem_mean_lower}
For $\mu,\nu\in\mathcal{P}_p(\mathcal{M})$ with finite first moments of their embeddings
$\tilde{\mu}=\phi_\#\mu$ and $\tilde{\nu}=\phi_\#\nu$ in $\mathcal{E}$,
\begin{equation}
\label{eq:pem_mean_lower}
\mathrm{PEMSW}_p(\mu,\nu)^p
\ge
\int_{S_{\mathcal{E}}}
\bigl|\langle a,\ m_{\tilde{\mu}}-m_{\tilde{\nu}}\rangle_{\mathcal{E}}\bigr|^p\,
\mathrm{d}\lambda_{\mathcal{E}}(a)
=
\kappa_{D,p}\,\|m_{\tilde{\mu}}-m_{\tilde{\nu}}\|_{\mathcal{E}}^p,
\end{equation}
where $D=\dim(\mathcal{E})$ and $\kappa_{D,p}=\int_{S^{D-1}}|\theta_1|^p\,\mathrm{d}\lambda(\theta)$.
\end{proposition}

\cref{prop:pem_mean_lower} shows that $\mathrm{PEMSW}_p$ cannot ignore discrepancies between embedded first-order moments. Under PEMs, these embedded means correspond to Fréchet means on the manifold. In the CorSW instantiations below, this means that discrepancies between OLM- or LSM-induced Fréchet means necessarily contribute to the CorSW loss, while the full sliced discrepancy still compares complete projected distributions rather than only their means.

The second observation compares the averaged discrepancy with the largest discrepancy over all slicing directions.
While $\mathrm{PEMSW}_p$ averages one-dimensional Wasserstein distances over geodesic slices, some directions may reveal stronger distributional differences than others.
The following max-sliced form characterizes this largest projected discrepancy and relates it to both the averaged pullback SW distance and the full Wasserstein distance on $\mathcal{M}$.

\begin{proposition}[Max-sliced comparison]
\label[proposition]{prop:pem_maxsw}
The max-sliced pullback discrepancy can be defined by
\begin{equation}
\mathrm{PEM\text{-}MaxSW}_p(\mu,\nu)
=
\sup_{a\in S_{\mathcal{E}}}
W_p\bigl(t_a{}_\#\mu,\ t_a{}_\#\nu\bigr).
\end{equation}
Then, for all $\mu,\nu\in\mathcal{P}_p(\mathcal{M})$,
\begin{equation}
\mathrm{PEMSW}_p(\mu,\nu)
\le
\mathrm{PEM\text{-}MaxSW}_p(\mu,\nu)
\le
W_p^{\mathcal{M}}(\mu,\nu).
\end{equation}
\end{proposition}

\cref{prop:pem_maxsw} places averaged pullback slicing between the largest one-dimensional sliced discrepancy and the full Wasserstein distance on $\mathcal{M}$. The first inequality follows from averaging versus taking a supremum, and the second follows from the $1$-Lipschitz property of each slicing coordinate under the pullback metric.

\paragraph{\textbf{Practical approximation}.} In practice, the integral in \cref{eq:pem_sw_def} is approximated by Monte Carlo sampling. Given $L$ independent directions $a_1,\ldots,a_L$ sampled uniformly from $S_{\mathcal{E}}$, we use
\begin{equation}
\label{eq:pem_sw_mc}
\widehat{\mathrm{PEMSW}}_{p,L}^p(\mu,\nu)
=
\frac{1}{L}
\sum_{\ell=1}^{L}
W_p^p
\Bigl(
t_{a_{\ell}}{}_{\#}\mu,\,
t_{a_{\ell}}{}_{\#}\nu
\Bigr).
\end{equation}
For empirical measures with the same number of samples, each one-dimensional Wasserstein distance is computed by sorting the projected values and matching the corresponding order statistics. This keeps the computation lightweight and makes the discrepancy suitable as an auxiliary training objective.

\subsection{Correlation SW under OLM}
\label{subsec:corsw_olm}

We first instantiate the pullback slicing construction on $\cor{n}$ endowed with OLM, using the operations summarized in \cref{tab:cor_manifold_ops}.
Under the OLM, the pullback map is $\phi_o=\offlog:\cor{n}\to \mathrm{Hol}(n)$ and the embedding space is $\mathrm{Hol}(n)$.
Since $\offlog(I_n)=0$, we fix the reference point as $x_0=I_n$.
Substituting this OLM embedding into the general slicing coordinate in \cref{thm:pem_projection} gives the following explicit form.

\begin{corollary}[Geodesic projection under OLM]
\label[corollary]{cor:olm_geodesic_projection}
Let $\mathcal{M}=\cor{n}$, $\phi_o=\offlog$, and $x_0=I_n$.
For any direction $A\in S_{\mathrm{Hol}}$ and any $C\in\cor{n}$, the corresponding slicing coordinate is given by
\begin{equation}
\label{eq:olm_pem_coordinate_cor}
t^{o}_{A}(C)=\langle A,\offlog(C)\rangle_F.
\end{equation}
\end{corollary}

The coordinate $t^{o}_{A}(C)$ is the signed scalar projection of the OLM embedding of $C$ along the direction $A$.
Thus, geodesic slicing under OLM reduces to linear slicing of off-log representations in $\mathrm{Hol}(n)$.
This leads naturally to the definition of an SW discrepancy on $\cor{n}$ induced by the OLM geometry.

\begin{definition}[CorSW under OLM]
\label{def:corsw_olm}
Let $p\ge 1$ and $\mu,\nu\in\mathcal{P}_p(\cor{n})$.
The correlation SW distance under OLM is defined as
\begin{equation}
\label{eq:corsw_olm_def}
\mathrm{CorSW}^{o}_{p}(\mu,\nu)^{p}
=
\int_{S_{\mathrm{Hol}}}
W_p^{p}
\Bigl(
t^{o}_{A}{}_{\#}\mu,\,
t^{o}_{A}{}_{\#}\nu
\Bigr)
\,\mathrm{d}\lambda_{\mathrm{Hol}}(A),
\end{equation}
where
\begin{equation}
\label{eq:holm_unit_sphere}
S_{\mathrm{Hol}}
=
\{A\in\mathrm{Hol}(n):\|A\|_F=1\}
\end{equation}
is endowed with the uniform measure $\lambda_{\mathrm{Hol}}$.
\end{definition}

In empirical computation, $\mathrm{CorSW}^{o}_{p}$ is estimated by sampling directions from $S_{\mathrm{Hol}}$, projecting the points onto these directions, and computing one-dimensional Wasserstein distances between the projected samples.
Hence, the OLM geometry determines both the embedding representation and the admissible slicing directions.

\subsection{Correlation SW under LSM}
\label{subsec:corsw_lsm}

We then specialize the same pullback slicing construction to correlation manifolds $\cor{n}$ endowed with LSM.
Under the LSM, the pullback map is $\phi_{\star}=\lslog:\cor{n}\to \mathrm{Row}_0(n)$ and the embedding space is $\mathrm{Row}_0(n)$.
Since $\lslog(I_n)=0$, we again fix the reference point as $x_0=I_n$.
Compared with OLM, LSM uses a different embedding map and therefore gives a different linear representation of the same correlation matrix.
Substituting the LSM embedding into \cref{thm:pem_projection} yields the following slicing coordinate.

\begin{corollary}[Geodesic projection under LSM]
\label[corollary]{cor:lsm_geodesic_projection}
Let $\mathcal{M}=\cor{n}$, $\phi_{\star}=\lslog$, and $x_0=I_n$.
For any direction $A\in S_{\mathrm{Row}}$ and any $C\in\cor{n}$, the slicing coordinate is given by
\begin{equation}
\label{eq:lsm_pem_coordinate_cor}
t^{\star}_{A}(C)=\langle A,\lslog(C)\rangle_F.
\end{equation}
\end{corollary}

The coordinate $t^{\star}_{A}(C)$ is the signed scalar projection of the LSM embedding of $C$ along the direction $A$.
Thus, geodesic slicing under LSM reduces to linear slicing of log-scaled representations in $\mathrm{Row}_0(n)$.
The corresponding LSM-induced CorSW discrepancy is defined as follows.

\begin{figure}[t]
    \centering    \includegraphics[width=0.99\linewidth]{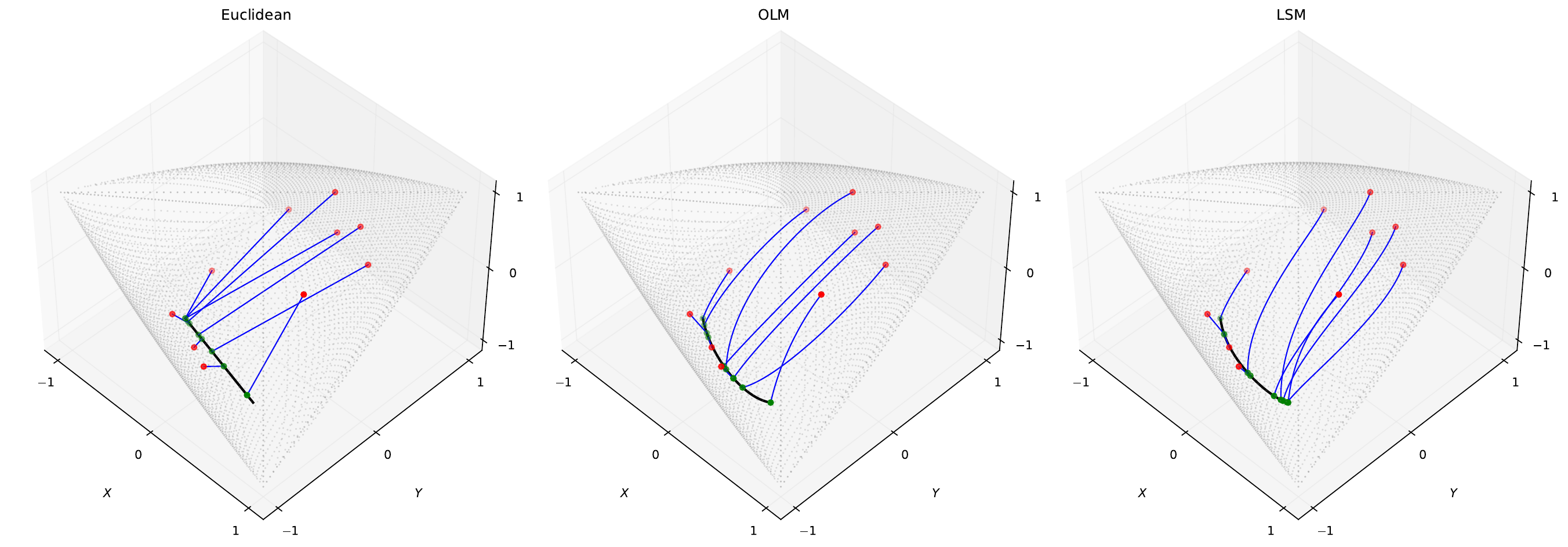}
    \caption{Visualization of geodesic and Euclidean projections of a $3\times3$ correlation matrix. Geodesic projections under OLM and LSM are compared with the Euclidean projection.}
    \label{fig:cor_proj_vis}
\end{figure}

\begin{definition}[CorSW under LSM]
\label{def:corsw_lsm}
Let $p\ge 1$ and $\mu,\nu\in\mathcal{P}_p(\cor{n})$.
The correlation SW distance under LSM is defined by
\begin{equation}
\label{eq:corsw_lsm_def}
\mathrm{CorSW}^{\star}_{p}(\mu,\nu)^{p}
=
\int_{S_{\mathrm{Row}}}
W_p^{p}
\Bigl(
t^{\star}_{A}{}_{\#}\mu,\,
t^{\star}_{A}{}_{\#}\nu
\Bigr)
\,\mathrm{d}\lambda_{\mathrm{Row}}(A),
\end{equation}
where
\begin{equation}
\label{eq:row_unit_sphere}
S_{\mathrm{Row}}
=
\{A\in\mathrm{Row}_0(n):\|A\|_F=1\}
\end{equation}
is endowed with the uniform measure $\lambda_{\mathrm{Row}}$.
\end{definition}

Both $\mathrm{CorSW}^{o}_{p}$ and $\mathrm{CorSW}^{\star}_{p}$ follow the same computational principle: correlation matrices are first mapped to a geometry-specific embedding space, then sliced by random directions, and finally compared through one-dimensional Wasserstein distances.
The difference lies in the chosen correlation geometry, which determines whether the sliced representation is given by $\offlog$ or $\lslog$.
To provide geometric intuition, \cref{fig:cor_proj_vis} visualizes geodesic projections of a $3\times3$ correlation matrix under OLM and LSM metrics, together with Euclidean projections, illustrating how different geometries induce distinct projection behaviors on correlation manifolds.

\section{Experiments}
In this section, we evaluate the proposed CorSW framework across three EEG decoding paradigms, examining its predictive performance, robustness to cross-session domain shifts, key design choices, training efficiency, and neurophysiological interpretability.

\subsection{Domain Generalization with CorSW}

\DontPrintSemicolon
\begin{algorithm}[t]
\caption{Correlation SW domain generalization loss}
\label{alg:cswd}
\KwIn{
batch $\mathcal{B}=\{C_i\in\cor{n},\ d(i)\in\mathcal{D}_{\mathcal{B}}\}_{i=1}^M$;
$g\in\{\mathrm{o},\star\}$;
number of slices $S$;
order $p=2$
}

Initialize $\mathcal{L}_{\mathrm{swd}}\leftarrow 0$\;

\For{each domain $d \in \mathcal{D}_{\mathcal{B}}$}{

    $\mathcal{C}^d \leftarrow \{C_i \mid d(i)=d\}$

    $\widetilde{\mathbf{Z}}^d \sim \mathcal{N}(0,I)\ \text{with shape}(\mathcal{C}^d)$ 

    $\mathbf{Z}^d \leftarrow \Pi_{\mathcal{E}_g}(\widetilde{\mathbf{Z}}^d)$ // Project to $\mathcal{E}_g$

    $\mathbf{Q}^d \leftarrow
    \Exp^{g} (\frac{\mathbf{Z}^d}{\|\mathbf{Z}^d\|_\mathrm{F}+\epsilon})
    \subset\cor{n}$ // Reference distribution

    Initialize $\ell^d \leftarrow 0$

    \For{$s=1$ \KwTo $S$}{
        $A_s \sim \mathrm{Unif}(\mathbb{S}_{\mathcal{E}})$ // Uniform unit slicing direction
        
        $u_s^{d} \leftarrow t_{A_s}^{g}(\mathcal{C}^d), \; v_s^{d} \leftarrow t_{A_s}^{g}(\mathcal{Q}^d)$ // 1D slicing along $A_s$
        
        $\ell^d \leftarrow \ell^d + W_p^p\!\left(u_s^{d},\, v_s^{d}\right)$
    }
    $\mathcal{L}_{\mathrm{swd}} \leftarrow \mathcal{L}_{\mathrm{swd}} + \ell^d / S$
}
\Return $\mathcal{L}_{\mathrm{swd}}$
\end{algorithm}

Following~\cite{li2026heegnet}, we adopt CorSW as a source-only domain generalization (DG) objective by aligning correlation representations of each source domain to a common reference law.
For a source domain $d$, let $\mathcal{C}^d \subset \cor{n}$ denote the empirical distribution of its feature-level correlation matrices.
Rather than matching all source domains pairwise, CorSW aligns each $\mathcal{C}^d$ to the same domain-agnostic reference law, which avoids selecting a particular source as the alignment target and keeps the loss linear in the number of source domains. This reference-based formulation is well suited to source-only DG, since it avoids treating any observed source domain as a surrogate target for unseen domains.
Moreover, constructing the reference law in the geometry-specific embedding space ensures that the sampled target distribution is compatible with the selected OLM or LSM geometry after being mapped back to $\cor{n}$.

The reference samples are generated in the geometry-specific embedding subspace $\mathcal{E}_g$, with $\mathcal{E}_{\mathrm{o}}=\mathrm{Hol}(n)$ for OLM and $\mathcal{E}_{\star}=\mathrm{Row}_0(n)$ for LSM.
Specifically, Gaussian matrix samples are first projected onto $\mathcal{E}_g$ and then mapped back to $\cor{n}$ through $\Exp^g$, ensuring valid inputs for the selected geometry.
In each mini-batch, an independent reference set $\mathcal{Q}^d$ with the same cardinality as $\mathcal{C}^d$ is sampled for stable one-dimensional Wasserstein matching. The discrepancy between $\mathcal{C}^d$ and $\mathcal{Q}^d$ is measured by CorSW through random geodesic slicing.
For each direction $A_s\sim\mathrm{Unif}(\mathbb{S}_{\mathcal{E}_g})$, both distributions are projected by $t_{A_s}^g$ and compared using the $p$-Wasserstein distance.
The final loss averages over slicing directions and sums across source domains, encouraging domain-invariant representations while preserving correlation-manifold geometry.
The overall procedure is summarized in~\cref{alg:cswd}.

\subsection{Datasets and Preprocessing}

We conduct experiments on three representative BCI paradigms that cover both time-asynchronous and time-synchronous EEG decoding scenarios, each exhibiting distinct neurophysiological characteristics.
Specifically, we evaluate motor imagery (MI) decoding on the BCIC-IV-2a dataset~\cite{brunner2008bci}, which represents a typical time-asynchronous setting, and steady-state visual evoked potential (SSVEP) decoding on the MAMEM-SSVEP-II dataset~\cite{mamem} as well as error-related negativity (ERN) decoding on the BCI-ERN dataset~\cite{bciern}, both of which correspond to time-synchronous EEG paradigms. This diverse experimental setup allows us to comprehensively assess the effectiveness and generality of the proposed method across different temporal structures and decoding regimes.

\paragraph{\textbf{MI}}
The BCIC-IV-2a dataset~\cite{brunner2008bci} is a widely used benchmark for MI decoding.
It consists of EEG recordings from nine subjects performing four-class MI tasks, namely left hand, right hand, feet, and tongue movements.
Each subject participated in two recording sessions, and each trial has a duration of four seconds.
The dataset contains 22 EEG channels and 3 EOG channels sampled at 250~Hz.
Classification accuracy is used as the evaluation metric.

\paragraph{\textbf{SSVEP}}
The MAMEM-SSVEP-II dataset~\cite{mamem} contains EEG recordings from eleven subjects, each recorded over five sessions.
In each session, subjects were instructed to visually attend to flickering stimuli at five target frequencies
(6.66, 7.50, 8.57, 10.00, and 12.00~Hz), yielding one trial per frequency.
Each session comprises five trials per subject, corresponding to the five stimulation frequencies.
We use the 1--5 second interval after stimulus onset and split each trial into four non-overlapping one-second segments for model input.
This produces 20 input segments per session and 100 segments per subject across all sessions.
We report classification accuracy.

\paragraph{\textbf{ERN}}
The BCI-ERN dataset~\cite{bciern} originates from a Kaggle BCI Challenge and includes EEG recordings from 26 participants engaged in a P300-based spelling task.
ERN signals are elicited following incorrect feedback, leading to a binary classification problem with a pronounced class imbalance.
This task evaluates the detection of erroneous feedback from event-related EEG responses.
To account for this imbalance, performance is evaluated using the area under the receiver operating characteristic curve (AUC).

\begin{table}[!t]
\centering
\caption{Average task performance ($\pm$ std., \%) over 10 runs for baseline and CorAtt models with and without CorSW.}
\label{tab:result_main}
\resizebox{\linewidth}{!}
{
\begin{tabular}{lccc}
\toprule
\textbf{Models} &  \textbf{MI} & \textbf{SSVEP} & \textbf{ERN} \\
\midrule
\rowcolor{gray!15} \multicolumn{4}{l}{ \small \textit{Euclidean methods}} \\
EEGNet~\cite{eegnet}          & $61.84 \pm 6.39$ & $53.72 \pm 7.23$  & $74.28 \pm 2.47$ \\
ShallowCNet~\cite{matt18}     & $57.43 \pm 6.25$ & $56.93 \pm 6.97$  & $71.86 \pm 2.64$ \\
EEG-TCNet~\cite{eeg-tcnet}       & $67.09 \pm 4.66$ & $55.45 \pm 7.66$  & $77.05 \pm 2.46$ \\
FBCNet~\cite{matt53}          & $56.52 \pm 3.07$ & $53.09 \pm 5.67$  & $60.47 \pm 3.06$ \\
MBEEGSE~\cite{matt50}         & $64.58 \pm 6.07$ & $56.45 \pm 7.27$  & $75.46 \pm 2.34$ \\
\midrule 
\rowcolor{gray!15} \multicolumn{4}{l}{ \small \textit{Riemannian methods}} \\
MAtt~\cite{matt}            & $74.71 \pm 5.01$ & $65.19 \pm 3.14$  & $75.68 \pm 2.23$ \\
GDLNet~\cite{gdlnet}          & $69.32 \pm 2.89$ & $65.52 \pm 2.86$  & $78.23 \pm 2.52$ \\
CorAtt-OLM~\cite{hu2025correlation} & $75.01 \pm 2.78$ & $67.39 \pm 3.22$ & $78.78 \pm 3.40$ \\
\textbf{CorAtt-OLM+CorSW-OLM } &  $75.49 \pm 2.34$ & $\underline{68.58 \pm 1.01}$ & $\mathbf{81.33 \pm 1.14}$ \\
CorAtt-LSM~\cite{hu2025correlation} & $74.47 \pm 2.43$ & $67.74 \pm 2.44$ & $78.63 \pm 3.31$ \\
\textbf{CorAtt-LSM+CorSW-LSM }  & $74.98 \pm 2.67$ & $\mathbf{68.60 \pm 1.62}$ & $\underline{80.07 \pm 0.89}$ \\
CorAtt-MIX~\cite{hu2025correlation} & $\underline{75.56 \pm 1.58}$ & $68.27 \pm 2.50$ & $79.04 \pm 2.91$ \\
\textbf{CorAtt-MIX+CorSW-OLM } & $\mathbf{75.58 \pm 1.23}$ & $68.51 \pm 1.30$ & $79.90 \pm 2.44$ \\
\bottomrule
\end{tabular}
}
\end{table}

\begin{table*}[!t]
\centering
\caption{Cross-session DG results (\%) on MI, SSVEP, and ERN tasks.}
\resizebox{\linewidth}{!}{
\begin{tabular}{l cc ccccc ccccc c}
\toprule
\multirow{2}{*}{\textbf{Models}}
& \multicolumn{2}{c}{\textbf{MI}}
& \multicolumn{5}{c}{\textbf{SSVEP}}
& \multicolumn{5}{c}{\textbf{ERN}} & \multirow{2}{*}{\textbf{Mean}} \\
\cmidrule(lr){2-3} \cmidrule(lr){4-8} \cmidrule(lr){9-13}
& S1$\rightarrow$S2 & S2$\rightarrow$S1
& S1234$\rightarrow$S5 & S2345$\rightarrow$S1 & S1345$\rightarrow$S2 & S1245$\rightarrow$S3 & S1235$\rightarrow$S4
& S1234$\rightarrow$S5 & S2345$\rightarrow$S1 & S1345$\rightarrow$S2 & S1245$\rightarrow$S3 & S1235$\rightarrow$S4 \\
\midrule
CorAtt-OLM
& $70.53 \pm 2.12$ & $74.71 \pm 2.40$ & $67.13 \pm 0.67$ & $66.55 \pm 1.05$ & $66.34 \pm 0.76$ & $67.94 \pm 0.79$ & $\mathbf{69.79 \pm 0.85}$ & $76.67 \pm 0.44$ & $79.73 \pm 0.91$ & $81.77 \pm 0.82$ & $80.56 \pm 0.36$ & $81.07 \pm 0.53$
& $73.57$ \\
+ IW~\cite{Choi_2021_CVPR}
& $70.85 \pm 2.44$ & $74.95 \pm 2.06$ & $67.58 \pm 0.71$ & $67.02 \pm 1.30$ & \underline{$67.05 \pm 0.68$} & \underline{$68.51 \pm 1.01$} & $68.00 \pm 0.87$ & \underline{$77.35 \pm 0.37$} & \underline{$82.44 \pm 1.09$} & $84.61 \pm 0.73$ & \underline{$83.50 \pm 0.37$} & $83.73 \pm 0.64$
& $74.63$ \\
+ DAC-SC~\cite{lee2023decompose}
& $69.86 \pm 2.54$ & $74.18 \pm 2.25$ & $66.39 \pm 0.71$ & $65.96 \pm 0.92$ & $65.83 \pm 0.89$ & $67.67 \pm 0.94$ & $69.43 \pm 0.90$ & $76.00 \pm 0.48$ & $81.58 \pm 0.95$ & $83.36 \pm 0.91$ & $81.66 \pm 0.45$ & $82.61 \pm 0.66$
& $73.71$ \\
+ RIW~\cite{lee2023decompose}
& $70.64 \pm 2.42$ & $74.91 \pm 2.36$ & $67.34 \pm 0.75$ & $66.83 \pm 1.15$ & $66.60 \pm 0.67$ & $68.11 \pm 0.98$ & $68.76 \pm 0.75$ & $76.75 \pm 0.50$ & $82.25 \pm 1.02$ & $84.37 \pm 0.80$ & $83.16 \pm 0.43$ & $83.65 \pm 0.56$
& $74.45$ \\
+ DRIW~\cite{bi2025learning}
& \underline{$70.96 \pm 2.55$} & \underline{$75.01 \pm 2.25$} & \underline{$67.64 \pm 0.76$} & \underline{$67.04 \pm 1.07$} & $66.82 \pm 0.83$ & $68.42 \pm 0.99$ & \underline{$69.44 \pm 0.75$} & $77.15 \pm 0.54$ & \underline{$82.44 \pm 0.94$} & \underline{$84.66 \pm 0.83$} & \underline{$83.50 \pm 0.44$} & \underline{$83.86 \pm 0.59$}
& \underline{$74.75$} \\
+ \textbf{CorSW-OLM }
& $\mathbf{71.31 \pm 2.26}$ & $\mathbf{75.49 \pm 2.34}$ & $\mathbf{68.58 \pm 1.01}$ & $\mathbf{68.51 \pm 1.00}$ & $\mathbf{70.26 \pm 1.05}$ & $\mathbf{71.49 \pm 0.75}$ & $69.05 \pm 1.77$ & $\mathbf{81.33 \pm 1.14}$ & $\mathbf{82.49 \pm 2.67}$ & $\mathbf{85.03 \pm 0.41}$ & $\mathbf{84.55 \pm 0.25}$ & $\mathbf{84.11 \pm 0.87}$
& $\mathbf{76.02}$ \\
& {\color{gain}$+0.78$} & {\color{gain}$+0.78$} & {\color{gain}$+1.45$} & {\color{gain}$+1.96$} & {\color{gain}$+3.92$} & {\color{gain}$+3.55$} & {\color{gain}$-0.74$} & {\color{gain}$+4.66$} & {\color{gain}$+2.76$} & {\color{gain}$+3.26$} & {\color{gain}$+3.99$} & {\color{gain}$+3.04$} & {\color{gain}$+2.45$} \\
\midrule
CorAtt-LSM
& $69.72 \pm 2.42$ & $73.46 \pm 2.15$ & $67.52 \pm 1.30$ & $66.88 \pm 0.96$ & $66.91 \pm 1.05$ & $68.35 \pm 0.52$ & $69.89 \pm 1.02$ & $76.49 \pm 0.82$ & $79.35 \pm 0.61$ & $82.26 \pm 0.36$ & $80.40 \pm 0.28$ & $80.47 \pm 0.61$
& $73.47$ \\
+ IW~\cite{Choi_2021_CVPR}
& \underline{$70.08 \pm 2.28$} & $73.85 \pm 2.19$ & $67.81 \pm 1.20$ & $67.11 \pm 1.02$ & $67.12 \pm 1.12$ & $68.67 \pm 0.46$ & \underline{$70.31 \pm 1.12$} & $76.94 \pm 0.90$ & $82.06 \pm 0.55$ & $84.84 \pm 0.40$ & $82.95 \pm 0.31$ & $83.12 \pm 0.55$
& $74.57$ \\
+ DAC-SC~\cite{lee2023decompose}
& $69.01 \pm 2.67$ & $72.96 \pm 2.01$ & $66.87 \pm 1.37$ & $66.07 \pm 1.09$ & $66.28 \pm 1.06$ & $67.65 \pm 0.53$ & $69.29 \pm 0.96$ & $75.69 \pm 0.79$ & $81.22 \pm 0.57$ & $84.03 \pm 0.31$ & $82.20 \pm 0.28$ & $82.29 \pm 0.55$
& $73.63$ \\
+ RIW~\cite{lee2023decompose}
& $69.94 \pm 2.66$ & $73.69 \pm 2.12$ & $67.71 \pm 1.28$ & $67.05 \pm 0.98$ & $67.09 \pm 1.21$ & $68.55 \pm 0.53$ & $70.19 \pm 0.95$ & $76.83 \pm 0.90$ & $81.96 \pm 0.58$ & $84.79 \pm 0.36$ & $82.88 \pm 0.33$ & $83.06 \pm 0.64$
& $74.48$ \\
+ DRIW~\cite{bi2025learning}
& $70.03 \pm 2.18$ & $\mathbf{74.01 \pm 2.30}$ & \underline{$67.99 \pm 1.33$} & \underline{$67.28 \pm 1.04$} & \underline{$67.33 \pm 0.97$} & \underline{$68.88 \pm 0.45$} & $\mathbf{70.53 \pm 1.07}$ & \underline{$77.18 \pm 0.94$} & $\mathbf{82.17 \pm 0.66}$ & \underline{$85.03 \pm 0.38$} & \underline{$83.12 \pm 0.29$} & \underline{$83.29 \pm 0.70$}
& \underline{$74.74$} \\
+ \textbf{CorSW-LSM }
& $\mathbf{70.23 \pm 2.87}$ & \underline{$74.98 \pm 2.67$} & $\mathbf{68.60 \pm 1.62}$ & $\mathbf{68.54 \pm 0.66}$ & $\mathbf{69.48 \pm 1.19}$ & $\mathbf{72.23 \pm 0.27}$ & $68.16 \pm 0.99$ & $\mathbf{80.07 \pm 0.89}$ & \underline{$82.07 \pm 2.04$} & $\mathbf{85.71 \pm 0.88}$ & $\mathbf{84.31 \pm 1.27}$ & $\mathbf{83.36 \pm 0.48}$
& $\mathbf{75.64}$ \\
& {\color{gain}$+0.51$} & {\color{gain}$+1.52$} & {\color{gain}$+1.08$} & {\color{gain}$+1.66$} & {\color{gain}$+2.57$} & {\color{gain}$+3.88$} & {\color{gain}$-1.73$} & {\color{gain}$+3.58$} & {\color{gain}$+2.72$} & {\color{gain}$+3.45$} & {\color{gain}$+3.91$} & {\color{gain}$+2.89$} & {\color{gain}$+2.17$} \\
\midrule
CorAtt-MIX
& $70.12 \pm 2.26$ & $75.46 \pm 2.16$ & $68.20 \pm 1.16$ & $66.42 \pm 0.87$ & $66.33 \pm 0.80$ & $68.45 \pm 1.05$ & $69.61 \pm 0.98$ & $75.60 \pm 0.32$ & $79.14 \pm 3.09$ & $81.56 \pm 0.34$ & $80.30 \pm 0.17$ & $79.74 \pm 0.43$
& $73.41$ \\
+ IW~\cite{Choi_2021_CVPR}
& $70.26 \pm 2.52$ & $75.64 \pm 2.18$ & $68.35 \pm 1.07$ & $66.62 \pm 0.77$ & $66.55 \pm 0.73$ & $68.74 \pm 1.28$ & $69.80 \pm 0.94$ & $75.99 \pm 0.37$ & $81.87 \pm 2.75$ & $84.10 \pm 0.38$ & $82.88 \pm 0.20$ & $82.33 \pm 0.45$
& $74.43$ \\
+ DAC-SC~\cite{lee2023decompose}
& $69.44 \pm 2.35$ & $74.77 \pm 2.43$ & $67.55 \pm 1.23$ & $65.85 \pm 0.83$ & $65.80 \pm 0.85$ & $67.76 \pm 1.21$ & $68.93 \pm 0.99$ & $74.92 \pm 0.35$ & $80.88 \pm 2.93$ & $83.39 \pm 0.35$ & $82.12 \pm 0.19$ & $81.41 \pm 0.47$
& $73.57$ \\
+ RIW~\cite{lee2023decompose}
& $70.28 \pm 2.37$ & \underline{$75.71 \pm 2.50$} & \underline{$68.43 \pm 1.29$} & $66.65 \pm 0.85$ & $66.56 \pm 0.95$ & $68.82 \pm 1.12$ & \underline{$69.86 \pm 1.04$} & $75.98 \pm 0.35$ & \underline{$81.93 \pm 2.61$} & $84.13 \pm 0.39$ & $82.93 \pm 0.18$ & \underline{$82.40 \pm 0.49$}
& $74.47$ \\
+ DRIW~\cite{bi2025learning}
& \underline{$70.54 \pm 2.43$} & $\mathbf{75.96 \pm 2.36}$ & $\mathbf{68.64 \pm 1.24}$ & \underline{$66.84 \pm 0.79$} & \underline{$66.77 \pm 0.73$} & \underline{$69.00 \pm 1.11$} & $\mathbf{70.06 \pm 1.13}$ & \underline{$76.25 \pm 0.36$} & $\mathbf{82.03 \pm 2.67}$ & \underline{$84.33 \pm 0.40$} & \underline{$83.08 \pm 0.22$} & $\mathbf{82.56 \pm 0.53}$
& \underline{$74.67$} \\
+ \textbf{CorSW-OLM }
& $\mathbf{70.88 \pm 2.26}$ & $75.58 \pm 2.16$ & $68.51 \pm 1.30$ & $\mathbf{67.95 \pm 0.59}$ & $\mathbf{69.71 \pm 0.67}$ & $\mathbf{71.73 \pm 0.31}$ & $68.26 \pm 0.52$ & $\mathbf{79.90 \pm 1.11}$ & $81.66 \pm 1.39$ & $\mathbf{85.60 \pm 0.88}$ & $\mathbf{83.38 \pm 1.47}$ & $82.24 \pm 0.69$
& $\mathbf{75.45}$ \\
& {\color{gain}$+0.76$} & {\color{gain}$+0.12$} & {\color{gain}$+0.31$} & {\color{gain}$+1.53$} & {\color{gain}$+3.38$} & {\color{gain}$+3.28$} & {\color{gain}$-1.35$} & {\color{gain}$+4.30$} & {\color{gain}$+2.52$} & {\color{gain}$+4.04$} & {\color{gain}$+3.08$} & {\color{gain}$+2.50$} & {\color{gain}$+2.04$} \\
\bottomrule
\end{tabular}
}
\label{tab:dg_results}
\end{table*}

\begin{table}[!t]
\centering
\caption{Cross-session comparison with manifold-based EEG adaptation baselines.}
\label{tab:manifold_eeg_baselines}
\resizebox{\linewidth}{!}{
\begin{tabular}{lccc}
\toprule
\textbf{Method} & \textbf{MI} & \textbf{SSVEP} & \textbf{ERN} \\
\midrule
SPDDSMBN~\cite{kobler2022spd} & $69.38 \pm 2.11$ & $64.72 \pm 0.95$ & $74.86 \pm 0.88$ \\
Deep OT-SPD~\cite{ju2022deep} & $69.91 \pm 2.35$ & $63.85 \pm 1.02$ & $75.43 \pm 1.12$ \\
SPDIM~\cite{ICLR2025_fc6cd575} & $70.52 \pm 2.08$ & $65.31 \pm 0.89$ & $76.12 \pm 0.94$ \\
CorAtt-OLM & \underline{$72.62 \pm 2.26$} & \underline{$67.15 \pm 0.82$} & \underline{$80.78 \pm 0.66$} \\
CorAtt+CorSW & $\mathbf{73.40 \pm 2.30}$ & $\mathbf{69.78 \pm 1.15}$ & $\mathbf{83.30 \pm 1.05}$ \\
\bottomrule
\end{tabular}
}
\end{table}

\subsection{Experimental Setup}

\paragraph{\textbf{Evaluation Protocols}}
We evaluate both predictive performance and robustness to domain shifts under two cross-session protocols.

\emph{(1) Standard cross-session evaluation.}
Following prior works~\cite{matt,hu2025correlation}, we adopt task-specific cross-session splits.
For MI, session 1 is used for training, with one-eighth (12.5\%) of trials reserved for validation, and session 2 is used for testing.
For SSVEP and ERN, sessions 1--3 are used for training, session 4 for validation, and session 5 for testing.
This protocol is used in \cref{tab:result_main} for comparison with Euclidean and Riemannian baselines.

\emph{(2) Exhaustive cross-session DG evaluation.}
We further enumerate multiple train$\rightarrow$test session splits to evaluate robustness under session-wise shifts.
We denote by $S_T \rightarrow S_E$ a transfer setting where the model is trained on sessions $T$ and evaluated on held-out sessions $E$.
For MI, we report $S1 \rightarrow S2$ and $S2 \rightarrow S1$.
For SSVEP and ERN, we use leave-one-session-out transfers: $S1234 \rightarrow S5$, $S2345 \rightarrow S1$, $S1345 \rightarrow S2$, $S1245 \rightarrow S3$, and $S1235 \rightarrow S4$.
\paragraph{\textbf{Baselines}}
We consider a diverse set of baselines.
Specifically, we include widely used Euclidean deep learning models for EEG decoding, including
ShallowConvNet~\cite{matt18}, EEGNet~\cite{eegnet}, MBEEGSE~\cite{matt50}, and FBCNet~\cite{matt53}.
To highlight the benefit of correlation geometry, we further compare with representative Riemannian and geometric deep learning approaches, namely MAtt~\cite{matt} and GDLNet~\cite{gdlnet}.
For DG, we additionally compare with several classical instance-level alignment methods, including IW~\cite{Choi_2021_CVPR}, DAC-SC~\cite{lee2023decompose}, RIW~\cite{bi2024learning}, and DRIW~\cite{bi2025learning}.
We also report targeted comparisons with manifold-based domain adaptation methods, including SPDDSMBN~\cite{kobler2022spd}, Deep OT-SPD~\cite{ju2022deep}, and SPDIM~\cite{ICLR2025_fc6cd575}.

\paragraph{\textbf{Backbone and CorSW Configuration}.}
We adopt CorAtt~\cite{hu2025correlation} as the correlation-manifold backbone and apply CorSW to the correlation representations produced by its attention module, where session shifts are represented on $\cor{n}$.
We evaluate three geometric instantiations: CorAtt-OLM and CorAtt-LSM, which use OLM and LSM throughout the network, and CorAtt-MIX, which uses OLM in the attention module and LSM in the classification layer.
Accordingly, the CorSW geometry is matched to the representation space where the loss is imposed: OLM for CorAtt-OLM and CorAtt-MIX, and LSM for CorAtt-LSM.
The number of slicing directions is set to $S=400$, and the Wasserstein order is fixed to $p=2$.

\paragraph{\textbf{Implementation Details}}
For the BCIC-IV-2a dataset, the number of subparts, the size of the transformation matrix in CorAtt, the learning rate, and the batch size are set to $3$, $25\times25$, $5\times10^{-4}$, and $128$, respectively.
For the MAMEM-SSVEP-II dataset, these values are configured as $7$, $15\times15$, $5\times10^{-3}$, and $64$.
For the BCI-ERN dataset, we use $3$ subparts, a transformation matrix of size $14\times14$, a learning rate of $1\times10^{-3}$, and a batch size of $32$.
All models are trained using the Adam optimizer~\cite{kingma2017adammethodstochasticoptimization} with identical training schedules.

\subsection{Performance Comparison}

\begin{table*}[!t]
\centering
\captionsetup{font=small, labelfont=bf, skip=4pt}

\begin{minipage}[t]{0.315\textwidth}
\centering
\caption{Effect of reference construction.}
\label{tab:corsw_reference_ablation}
\footnotesize
\setlength{\tabcolsep}{3pt}
\begin{tabularx}{\linewidth}{@{}lYY@{}}
\toprule
\textbf{Reference} & \textbf{SSVEP} & \textbf{ERN} \\
\midrule
Gaussian (random)  & $\mathbf{68.58 \pm 1.01}$ & $\mathbf{81.33 \pm 1.14}$ \\
Gaussian (fixed)   & $68.07 \pm 1.45$ & $80.76 \pm 1.29$ \\
Gaussian (matched) & $68.24 \pm 1.08$ & $81.03 \pm 1.17$ \\
Source-mix         & $67.73 \pm 1.18$ & $80.47 \pm 1.21$ \\
\bottomrule
\end{tabularx}
\end{minipage}
\hfill
\begin{minipage}[t]{0.315\textwidth}
\centering
\caption{Effect of correlation embedding.}
\label{tab:corsw_embedding_ablation}
\footnotesize
\setlength{\tabcolsep}{3pt}
\begin{tabularx}{\linewidth}{@{}lYY@{}}
\toprule
\textbf{Embedding} & \textbf{SSVEP} & \textbf{ERN} \\
\midrule
Lower-tri. vec.  & $65.58 \pm 0.68$ & $74.63 \pm 1.72$ \\
Full-matrix vec. & $66.49 \pm 0.71$ & $75.15 \pm 1.21$ \\
$\offlog$        & $\mathbf{68.58 \pm 1.01}$ & $\mathbf{81.33 \pm 1.14}$ \\
$\lslog$         & $67.74 \pm 2.44$ & $78.63 \pm 3.31$ \\
\bottomrule
\end{tabularx}
\end{minipage}
\hfill
\begin{minipage}[t]{0.315\textwidth}
\centering
\caption{Effect of slicing directions.}
\label{tab:corsw_slicing_ablation}
\footnotesize
\setlength{\tabcolsep}{3pt}
\begin{tabularx}{\linewidth}{@{}lYY@{}}
\toprule
\textbf{Direction} & \textbf{SSVEP} & \textbf{ERN} \\
\midrule
Random    & $67.76 \pm 1.07$ & $80.47 \pm 1.18$ \\
Learnable & $68.31 \pm 1.04$ & $81.02 \pm 1.16$ \\
Uniform    & $\mathbf{68.58 \pm 1.01}$ & $\mathbf{81.33 \pm 1.14}$ \\
\bottomrule
\end{tabularx}
\end{minipage}

\end{table*}

\paragraph{\textbf{Comparison with Existing EEG Baselines}}
As shown in \cref{tab:result_main}, geometry-aware methods consistently outperform Euclidean baselines across the three EEG paradigms.
Euclidean models such as EEGNet~\cite{eegnet} and ShallowConvNet~\cite{matt18} operate in standard feature spaces and show clear performance gaps, especially on MI and SSVEP.
Riemannian models, including MAtt~\cite{matt} and GDLNet~\cite{gdlnet}, improve over these Euclidean baselines by exploiting second-order EEG geometry.
Among them, CorAtt~\cite{hu2025correlation} provides a strong correlation-manifold backbone.

Building on CorAtt, CorSW further improves performance across different geometric instantiations.
CorAtt-OLM+CorSW-OLM achieves the largest gain on ERN, CorAtt-LSM+CorSW-LSM performs best on SSVEP, and CorAtt-MIX+CorSW-OLM obtains the highest MI accuracy.
In addition to improving the mean score, CorSW often reduces run-to-run variability. For example, the standard deviation of CorAtt-OLM decreases from $3.22$ to $1.01$ on SSVEP and from $3.40$ to $1.14$ on ERN after adding CorSW.
These results indicate that CorSW provides complementary distribution-level regularization beyond the backbone architecture, improving training stability under session-wise shifts without changing the inference pipeline.

\paragraph{\textbf{Comparison with DG Strategies}}
\cref{tab:dg_results} evaluates robustness under the exhaustive cross-session DG protocol.
To ensure a controlled comparison, we fix CorAtt as the backbone and add different DG objectives, including instance-level alignment and whitening-based strategies.
The \textit{Mean} column averages performance over all 12 train$\rightarrow$test transfers.
Across OLM, LSM, and MIX backbones, CorSW achieves the highest mean score among the compared DG methods, improving CorAtt-OLM, CorAtt-LSM, and CorAtt-MIX by $+2.45$, $+2.17$, and $+2.04$, respectively.
The gains are especially pronounced on SSVEP and ERN, where the leave-one-session-out protocol exposes more diverse session-wise shifts.
Although CorSW is not uniformly better on every split, such as the SSVEP $S1235\rightarrow S4$ transfer, its improvements on the remaining transfers dominate the overall average.
On MI, the gains are smaller, which is expected because only two sessions are available and the evaluation contains fewer transfer directions.
Overall, these results show that aligning source-domain distributions on the correlation manifold is more effective than applying Euclidean alignment losses to correlation manifold-valued EEG representations.

\paragraph{\textbf{Comparison with Manifold-based EEG Adaptation}}
\cref{tab:manifold_eeg_baselines} further compares CorSW with manifold-based EEG domain adaptation (DA) baselines.
SPDDSMBN, Deep OT-SPD, and SPDIM are DA methods developed on the SPD manifold.
In contrast, CorSW performs distribution alignment directly on full-rank correlation matrices under correlation-manifold geometry.
Compared with these manifold-based adaptation methods, CorAtt+CorSW achieves the best performance on all three tasks, improving over the strongest SPDIM baseline by $+2.88$ on MI, $+4.47$ on SSVEP, and $+7.18$ on ERN.
It also improves the CorAtt-OLM backbone from $72.62$ to $73.40$ on MI, from $67.15$ to $69.78$ on SSVEP, and from $80.78$ to $83.30$ on ERN.
These results suggest that directly aligning correlation-valued distributions provides a more suitable adaptation signal for correlation-based EEG representations than applying SPD manifold DA methods to covariance-based representations.

\begin{figure}[t]
    \centering
    \includegraphics[width=1\linewidth]{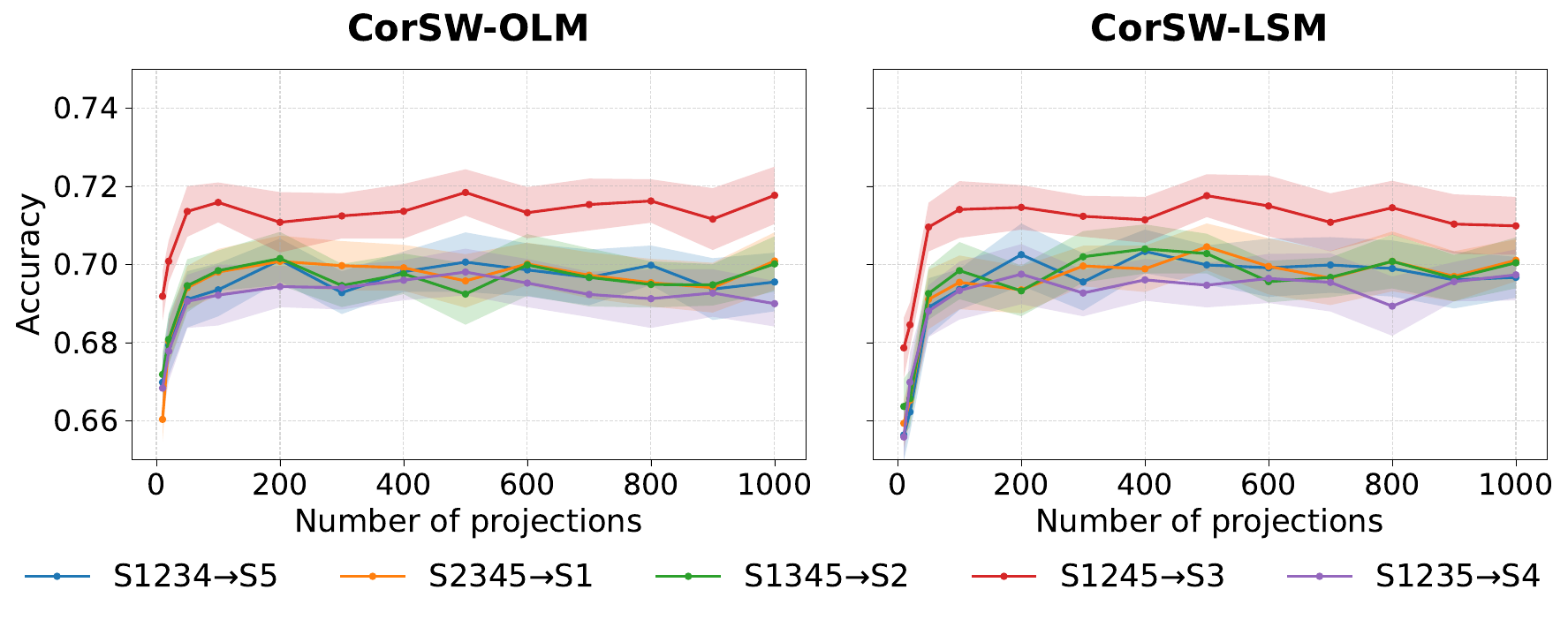}
    \caption{Performance of CorSW-OLM/LSM vs. slicing number $S$ on the MAMEM-SSVEP-II dataset.}
    \Description{Performance of CorSW-OLM/LSM vs. slicing number $S$ on the MAMEM-SSVEP-II dataset.}
    \label{fig:corsw_proj_num}
\end{figure}

\begin{figure*}[t]
    \centering
    \includegraphics[width=1\linewidth]{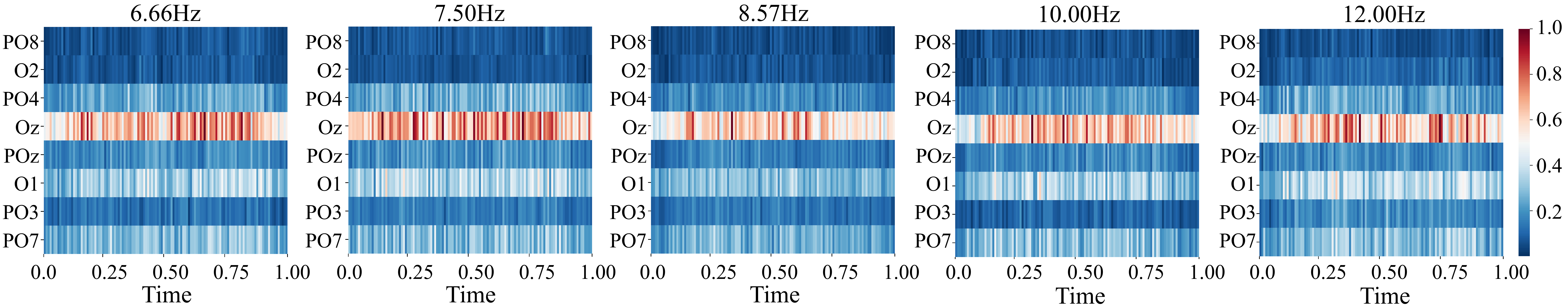}
    \caption{Heatmaps of CorAtt-OLM+CorSW-OLM for the S11 subject across five different frequencies on the MAMEM-SSVEP-II dataset. The x-axis and y-axis represent time and EEG channels, respectively.}
    \Description{TBD}
    \label{fig:mamem_heat}
\end{figure*}

\begin{figure*}[t]
    \centering
    \includegraphics[width=1\linewidth]{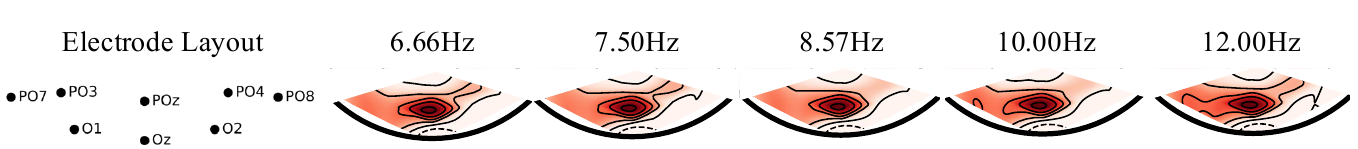}
    \caption{The diagram of electrode distribution (a) and the spatial topo-maps of CorAtt-OLM+CorSW-OLM for the S11 subject across five different frequencies on the MAMEM-SSVEP-II dataset. Strong gradient activations are marked in dark red.}
    \Description{TBD}
    \label{fig:mamem_topo}
\end{figure*}

\begin{figure*}[t]
    \centering
    \includegraphics[width=1\linewidth]{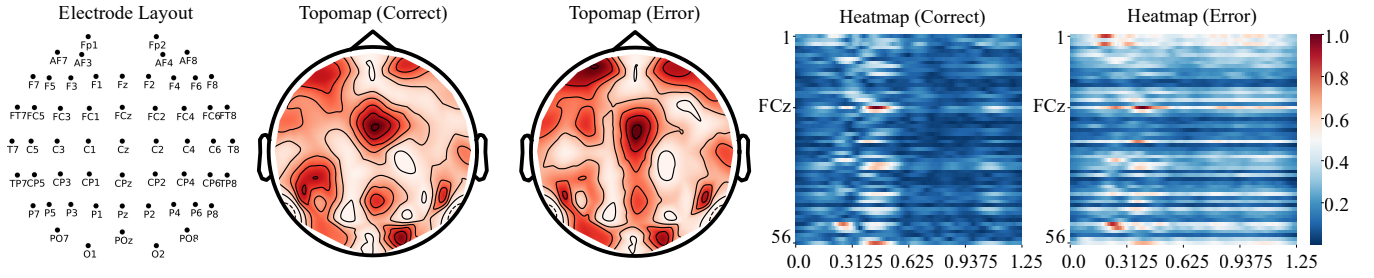}
    \caption{Visualization of heatmaps, topo-maps, and electrode distribution for subject S7 on the BCI-ERN dataset.}
    \Description{TBD}
    \label{fig:bcicha_topo_heat}
\end{figure*}

\subsection{Ablation Study}

\paragraph{\textbf{Reference Construction}.}
\cref{tab:corsw_reference_ablation} studies how the reference distribution $\mathcal{Q}^d$ is constructed. We compare four choices: Gaussian reference resampled at each iteration, fixed Gaussian sampled once before training, moment-matched Gaussian estimated from training features, and source-mix reference built from source-domain statistics. The resampled Gaussian performs best on SSVEP and ERN. Its consistent gains over the fixed and moment-matched Gaussians suggest that resampling provides mild stochastic regularization and prevents the model from overfitting to a single alignment target. The source-mix reference performs worst, indicating that directly reusing source statistics may introduce source-specific bias instead of providing a neutral target.

\paragraph{\textbf{Correlation Embedding}.}
\cref{tab:corsw_embedding_ablation} evaluates whether the geometry-aware embedding is necessary by replacing OLM/LSM mappings with flat SW baselines based on lower-triangular or full-matrix vectorization.
Both vectorization baselines substantially underperform the geometry-aware log embeddings, especially on ERN.
This confirms that CorSW benefits from slicing and Wasserstein matching in an embedding space that preserves the correlation-manifold geometry, rather than simply applying SW to matrix entries.

\paragraph{\textbf{Slicing Directions}.}
\cref{tab:corsw_slicing_ablation} compares different slicing direction strategies.
Uniform unit directions achieve the best performance on both tasks, while learnable directions add optimization complexity without further gains and non-uniform random directions are less stable.
Together with the sensitivity analysis in \cref{fig:corsw_proj_num}, these results support uniformly sampled directions as the default choice.

\paragraph{\textbf{Effect of the number of projections $S$.}}
We analyze the sensitivity of the number of slicing projections $S$ in \cref{fig:corsw_proj_num}.
CorSW-OLM and -LSM degrade when $S$ is small, indicating that too few directions lead to an inaccurate SW approximation on the correlation manifold.
As $S$ increases, performance improves rapidly and then saturates, suggesting that a moderate number of projections is sufficient for stable estimation.
This justifies using $S=400$ as the default setting, which offers a practical accuracy--efficiency trade-off.

\begin{table}[t]
\centering
\caption{Training time per iteration (ms/iter).}
\label{tab:train_time}
\scriptsize
\setlength{\tabcolsep}{2.2pt}
\renewcommand{\arraystretch}{0.99}
\resizebox{\columnwidth}{!}{%
\begin{tabular}{@{}lccccc@{}}
\toprule
Method & CorAtt & +IW & +DAC-SC & +DRIW & +CorSW \\
\midrule
Time & $27.50{\pm}0.66$ & $29.57{\pm}0.63$ & $32.10{\pm}0.88$ & $32.23{\pm}0.68$ & $32.73{\pm}0.97$ \\
\bottomrule
\end{tabular}%
}
\end{table}

\paragraph{\textbf{Training Efficiency}}
\cref{tab:train_time} reports the per-iteration training time under the same implementation setting. 
CorSW increases the training time from $27.50$ ms/iter to $32.73$ ms/iter, adding about $5.2$ ms per iteration. 
This overhead is comparable to DAC-SC and DRIW. 
The additional cost mainly comes from embedding correlation matrices, projecting them along sampled directions, and computing one-dimensional Wasserstein distances in the embedding space.
Since CorSW is used only as a training loss, it does not modify the backbone architecture or introduce any inference-time overhead.

\subsection{EEG Model Interpretation}

We further examine whether the learned correlation representations focus on neurophysiologically meaningful spatial--temporal patterns by visualizing gradient-based attribution maps of the trained CorAtt-OLM+CorSW-OLM model.
For the MAMEM task, as shown in \cref{fig:mamem_topo,fig:mamem_heat}, CorAtt-OLM+CorSW exhibits pronounced responses around the Oz electrode across all stimulus frequencies, particularly within the 0.25 to 0.75-second window corresponding to canonical SSVEP activity.
This spatial--temporal pattern is highly consistent with established neuroscience findings that SSVEP responses are dominated by the primary visual cortex~\cite{mamem-a1,mamem-a2}, indicating that the learned correlation representations capture meaningful visual processing dynamics rather than session-specific artifacts.

As shown in \cref{fig:bcicha_topo_heat}, for the BCI-ERN dataset, gradient responses distinguishing \emph{correct} and \emph{error} trials are primarily concentrated around the FCz electrode, with peak activations occurring between 0.1 and 0.4 seconds.
This observation aligns well with prior neurophysiological evidence that ERN originates from the anterior cingulate cortex and medial frontal regions involved in performance monitoring~\cite{bci-al}, suggesting that the model focuses on task-relevant neural signatures rather than session-specific artifacts.

\section{Conclusion}

In this paper, we studied Sliced-Wasserstein discrepancies on the manifold of full-rank correlation matrices by instantiating pullback slicing under two intrinsic correlation geometries, OLM and LSM.
The resulting CorSW-OLM and CorSW-LSM reduce geodesic slicing to linear slicing in geometry-specific embedding spaces, enabling closed-form slicing coordinates and efficient one-dimensional Wasserstein computations.
We further used CorSW as a source-only DG regularizer for EEG decoding by aligning source-domain correlation distributions to a common reference law.
Experiments on MI, SSVEP, and ERN benchmarks show that CorSW consistently improves correlation-manifold-based EEG decoding with low training overhead and no additional inference-time cost.

\section*{Limitations and Ethical Considerations}
A main limitation of this work is that the used PEMSW framework is restricted to manifolds endowed with PEMs. 
Although this setting covers some matrix manifolds, its direct extension to general manifolds remains unexplored. 
Additionally, this study uses existing EEG benchmarks and does not collect new human-subject data or introduce new clinical interventions. 
Nevertheless, EEG signals are sensitive biomedical data and may contain information related to cognitive states, health conditions, or individual identity. 
Any practical use of CorSW-based models should therefore follow strict data governance, including informed consent, secure storage, controlled access, and privacy protection. 
For applications such as brain-computer interfaces or clinical decision support, model predictions should not be treated as standalone medical decisions. 

\section*{GenAI Disclosure}
Generative AI tools were used to assist with language polishing and LaTeX formatting during manuscript preparation.
All technical content, mathematical formulations, experimental design, and results were conceived, implemented, and verified by the authors.

\begin{acks}
This work was supported in part by the Zhejiang Leading Innovative and Entrepreneur Team Introduction Program (2024R01007), and the “Pioneer” and “Leading Goose” Research and Development Program of Zhejiang (2025C02077), the National Natural Science Foundation of China (Grant Nos. 62306127 and 62332008), the National Key R\&D Program of China (Grant Nos. 2023YFF1105102 and 2023YFF1105105), the Natural Science Foundation of Jiangsu Province (BK20231040), the Fundamental Research Funds for the Central Universities (JUSRP124015), the Postgraduate Research \& Practice Innovation Program of Jiangsu Province (SJCX25\_1319). 
\end{acks}

\bibliographystyle{ACM-Reference-Format}
\balance
\bibliography{ref}

@article{david2019riemannian,
  title={A Riemannian Structure for {C}orrelation Matrices},
  author={David, Paul and Gu, Weiqing},
  journal={Operators and Matrices},
  volume={13},
  number={3},
  pages={607--627},
  year={2019},
}

@article{manifoldnet,
  title={{ManifoldNet}: A Deep Neural Network for Manifold-Valued Data With Applications},
  author={Chakraborty, Rudrasis and Bouza, Jose and Manton, Jonathan and Vemuri, Baba C},
  journal={IEEE Transactions on Pattern Analysis and Machine Intelligence},
  volume={44},
  number={2},
  pages={799--810},
  year={2022},
}

@article{chen2024adaptive,
  title={Adaptive {Log-Euclidean} Metrics for {SPD} Matrix Learning},
  author={Chen, Ziheng and Song, Yue and Xu, Tianyang and Huang, Zhiwu and Wu, Xiao-Jun and Sebe, Nicu},
  journal={IEEE Transactions on Image Processing},
  volume={33},
  pages={5194--5205},
  year={2024},
}

@inproceedings{chen2024rmlr,
  title={{RMLR}: Extending Multinomial Logistic Regression into General Geometries},
  author={Chen, Ziheng and Song, Yue and Wang, Rui and Wu, Xiaojun and Sebe, Nicu},
  booktitle={Advances in Neural Information Processing Systems},
  year={2024}
}

@inproceedings{ju2024deep,
  title={Deep Geodesic Canonical Correlation Analysis for Covariance-Based Neuroimaging Data},
  author={Ce Ju and Reinmar J Kobler and Liyao Tang and Cuntai Guan and Motoaki Kawanabe},
  booktitle={International Conference on Learning Representations},
  numpages={9},
  publisher={OpenReview.net},
  address={Vienna, Austria},
  year={2024},
}

@inproceedings{chen2024rmlrspd,
  title={{Riemannian} Multinomial Logistics Regression for {SPD} Neural Networks},
  author={Chen, Ziheng and Song, Yue and Liu, Gaowen and Kompella, Ramana Rao and Wu, Xiao-Jun and Sebe, Nicu},
  booktitle={Proceedings of the IEEE/CVF Conference on Computer Vision and Pattern Recognition},
  year={2024},
  pages={17086--17096}
}

@inproceedings{chen2023riemannian,
  title={Riemannian Local Mechanism for SPD Neural Networks},
  author={Chen, Ziheng and Xu, Tianyang and Wu, Xiao-Jun and Wang, Rui and Huang, Zhiwu and Kittler, Josef},
  booktitle={Proceedings of the AAAI Conference on Artificial Intelligence},
  year={2023}
}

@inproceedings{chen2024liebn,
  title={A Lie Group Approach to Riemannian Batch Normalization},
  author={Chen, Ziheng and Song, Yue and Liu, Yunmei and Sebe, Nicu},
  booktitle={International Conference on Learning Representations},
  publisher={OpenReview.net},
  address={Vienna, Austria},
  year={2024}
}

@article{eegnet,
  title={{EEGNet}: a compact convolutional neural network for {EEG}-based brain--computer interfaces},
  author={Lawhern, Vernon J and Solon, Amelia J and Waytowich, Nicholas R and Gordon, Stephen M and Hung, Chou P and Lance, Brent J},
  journal={Journal of Neural Engineering},
  volume={15},
  number={5},
  pages={056013},
  year={2018},
}

@article{matt18,
  title={Deep learning with convolutional neural networks for {EEG} decoding and visualization},
  author={Schirrmeister, Robin Tibor and Springenberg, Jost Tobias and Fiederer, Lukas Dominique Josef and Glasstetter, Martin and Eggensperger, Katharina and Tangermann, Michael and Hutter, Frank and Burgard, Wolfram and Ball, Tonio},
  journal={Human Brain Mapping},
  volume={38},
  number={11},
  pages={5391--5420},
  year={2017}
}

@article{matt53,
  title={{FBCNet}: A multi-view convolutional neural network for brain-computer interface},
  author={Mane, Ravikiran and Chew, Effie and Chua, Karen and Ang, Kai Keng and Robinson, Neethu and Vinod, A Prasad and Lee, Seong-Whan and Guan, Cuntai},
  journal={arXiv preprint},
  volume={arXiv:2104.01233},
  numpages={16},
  year={2021}
}

@article{matt50,
  title={A multi-branch convolutional neural network with squeeze-and-excitation attention blocks for {EEG}-based motor imagery signals classification},
  author={Altuwaijri, Ghadir Ali and Muhammad, Ghulam and Altaheri, Hamdi and Alsulaiman, Mansour},
  journal={Diagnostics},
  volume={12},
  number={4},
  pages={995},
  year={2022},
}

@inproceedings{matt,
  title={MAtt: A Manifold Attention Network for {EEG} Decoding},
  author={Pan, Yue-Ting and Chou, Jing-Lun and Wei, Chun-Shu},
  booktitle={Advances in Neural Information Processing Systems},
  volume={35},
  pages={31116--31129},
  year={2022},
}

@inproceedings{grnet,
  title={Building deep networks on {G}rassmann manifolds},
  author={Huang, Zhiwu and Wu, Jiqing and Van Gool, Luc},
  booktitle={Proceedings of the {AAAI} Conference on Artificial Intelligence},
  volume={32},
  pages={3279--3286},
  year={2018},
  publisher={AAAI Press},
  address={New Orleans, LA, USA}
}

@inproceedings{spdnetbn,
  title={{Riemannian} batch normalization for {SPD} neural networks},
  author={Brooks, Daniel and Schwander, Olivier and Barbaresco, Fr{\'e}d{\'e}ric and Schneider, Jean-Yves and Cord, Matthieu},
  booktitle={Advances in Neural Information Processing Systems},
  volume={32},
  pages={8890--8901},
  year={2019},
  publisher={Curran Associates, Inc.},
  address={Vancouver, Canada}
}

@inproceedings{spdnet,
  title={A {R}iemannian network for {SPD} matrix learning},
  author={Huang, Zhiwu and Van Gool, Luc},
  booktitle={Proceedings of the Thirty-First AAAI Conference on Artificial Intelligence},
  volume={31},
  pages={2036--2042},
  year={2017},
}

@inproceedings{eeg-tcnet,
  title={{EEG-TCNet}: An accurate temporal convolutional network for embedded motor-imagery brain--machine interfaces},
  author={Ingolfsson, Thorir Mar and Hersche, Michael and Wang, Xiaying and Kobayashi, Nobuaki and Cavigelli, Lukas and Benini, Luca},
  booktitle={IEEE International Conference on Systems, Man, and Cybernetics},
  pages={2958--2965},
  year={2020},
}

@inproceedings{gdlnet,
  title={A {Grassmannian} Manifold Self-Attention Network for Signal Classification},
  author={Wang, Rui and Hu, Chen and Chen, Ziheng and Wu, Xiao-Jun and Song, Xiaoning},
  booktitle={Proceedings of the International Joint Conference on Artificial Intelligence},
  pages={5099--5107},
  year={2024},
}

@inproceedings{wang2025learning,
  title={Learning to Normalize on the {SPD} Manifold under {B}ures-{W}asserstein Geometry},
  author={Wang, Rui and Jin, Shaocheng and Chen, Ziheng and Luo, Xiaoqing and Wu, Xiao-Jun},
  booktitle={Proceedings of the IEEE/CVF Conference on Computer Vision and Pattern Recognition},
  pages={8289--8298},
  year={2025},
}

@article{mamem-a1,
  title={Human {EEG} Responses to 1--100 {Hz} Flicker: Resonance Phenomena in Visual Cortex and Their Potential Correlation to Cognitive Phenomena},
  author={Herrmann, Christoph S},
  journal={Experimental Brain Research},
  volume={137},
  number={3--4},
  pages={346--352},
  year={2001},
}

@article{mamem-a2,
  title={Highly Interactive Brain-Computer Interface Based on Flicker-Free Steady-State Motion Visual Evoked Potential},
  author={Han, Chengcheng and Xu, Guanghua and Xie, Jun and Chen, Chaoyang and Zhang, Sicong},
  journal={Scientific Reports},
  volume={8},
  number={1},
  pages={5835},
  year={2018},
}

@article{jalili2011constructing,
  title={Constructing Brain Functional Networks From {EEG}: Partial and Unpartial Correlations},
  author={Jalili, Mahdi and Knyazeva, Maria G},
  journal={Journal of Integrative Neuroscience},
  volume={10},
  number={2},
  pages={213--232},
  year={2011},
}

@article{bci-al,
  title={What We’ve Learned From Mistakes: Insights From Error-Related Brain Activity},
  author={Hajcak, Greg},
  journal={Current Directions in Psychological Science},
  volume={21},
  number={2},
  pages={101--106},
  year={2012},
}

@article{epskamp2018tutorial,
  title={A Tutorial on Regularized Partial Correlation Networks},
  author={Epskamp, Sacha and Fried, Eiko I},
  journal={Psychological Methods},
  volume={23},
  number={4},
  pages={617--634},
  year={2018},
}

@article{bciern,
  title={Objective and Subjective Evaluation of Online Error Correction During P300-Based Spelling},
  author={Margaux, Perrin and Maby, Emmanuel and Daligault, S{\'e}bastien and Bertrand, Olivier and Mattout, J{\'e}r{\'e}mie},
  journal={Advances in Human-Computer Interaction},
  volume={2012},
  pages={1--13},
  year={2012},
}

@misc{mamem,
  title={MAMEM {EEG} {SSVEP} Dataset {II} (256 Channels, 11 Subjects, 5 Frequencies Presented Simultaneously)},
  author={Nikolopoulos, Spiros},
  year={2016},
  howpublished={Dataset},
  url={https://api.semanticscholar.org/CorpusID:63262529}
}

@inproceedings{kobler2022spd,
  title={SPD domain-specific batch normalization to crack interpretable unsupervised domain adaptation in {EEG}},
  author={Kobler, Reinmar and Hirayama, Jun-ichiro and Zhao, Qibin and Kawanabe, Motoaki},
  booktitle={Advances in Neural Information Processing Systems},
  volume={35},
  pages={6219--6235},
  year={2022},
}

@inproceedings{chen2025understanding,
  title={Understanding Matrix Function Normalizations in Covariance Pooling Through the Lens of {Riemannian} Geometry},
  author={Chen, Ziheng and Song, Yue and Wu, Xiao-Jun and Liu, Gaowen and Sebe, Nicu},
  booktitle={International Conference on Learning Representations},
  year={2025},
}

@article{ju2023graph,
  title={Graph Neural Networks on SPD Manifolds for Motor Imagery Classification: A Perspective From the Time-Frequency Analysis},
  author={Ju, Ce and Guan, Cuntai},
  journal={IEEE Transactions on Neural Networks and Learning Systems},
  volume={35},
  number={12},
  pages={17701--17715},
  year={2023},
}

@article{yair2019domain,
  title={Domain Adaptation with Optimal Transport on the Manifold of {SPD} Matrices},
  author={Yair, Or and Dietrich, Felix and Talmon, Ronen and Kevrekidis, Ioannis G},
  journal={arXiv preprint},
  volume={arXiv:1906.00616},
  numpages={12},
  year={2019},
  publisher={arXiv},
}

@article{thanwerdas2024permutation,
  title={Permutation-Invariant Log-Euclidean Geometries on Full-Rank Correlation Matrices},
  author={Thanwerdas, Yann},
  journal={SIAM Journal on Matrix Analysis and Applications},
  volume={45},
  number={2},
  pages={930--952},
  year={2024},
}

@article{brunner2008bci,
  title={BCI Competition 2008–Graz Data Set A},
  author={Brunner, Clemens and Leeb, Robert and Müller-Putz, Gernot and Schlögl, Alois and Pfurtscheller, Gert},
  journal={Institute for Knowledge Discovery (Laboratory of Brain-Computer Interfaces), Graz University of Technology},
  volume={},
  number={},
  pages={},
  year={2008}
}

@book{loring2011introduction,
  title={An Introduction to Manifolds},
  author={Tu, Loring W.},
  publisher={Springer},
  address={New York, NY, USA},
  year={2011},
  isbn={9781441973998}
}

@inproceedings{chen2025gyrogroup,
  title={Gyrogroup Batch Normalization},
  author={Chen, Ziheng and Song, Yue and Wu, Xiao-Jun and Sebe, Nicu},
  booktitle={International Conference on Learning Representations},
  year={2025},
}

@book{do1992riemannian,
  title={Riemannian Geometry},
  author={Do Carmo, Manfredo Perdig{\~a}o and Flaherty, Francis J.},
  publisher={Springer},
  address={New York, NY, USA},
  year={1992},
  isbn={9780387978780}
}

@article{archakov2021new,
  title={A New Parametrization of Correlation Matrices},
  author={Archakov, Ilya and Hansen, Peter Reinhard},
  journal={Econometrica},
  volume={89},
  number={4},
  pages={1699--1715},
  year={2021},
  doi={10.3982/ECTA16910}
}

@book{villani2009optimal,
  title={Optimal Transport: Old and New},
  author={Villani, C{\'e}dric},
  publisher={Springer},
  series={Grundlehren der Mathematischen Wissenschaften},
  volume={338},
  address={Berlin, Heidelberg, Germany},
  year={2009},
}

@inproceedings{paty2019subspace,
  title={Subspace Robust Wasserstein Distances},
  author={Paty, Fran{\c{c}}ois-Pierre and Cuturi, Marco},
  booktitle={Proceedings of the 36th International Conference on Machine Learning},
  volume={97},
  pages={5072--5081},
  year={2019},
  publisher={PMLR},
  doi={10.5555/3327144.3327253}
}

@article{peyre2019computational,
  title={Computational Optimal Transport: With Applications to Data Science},
  author={Peyr{\'e}, Gabriel and Cuturi, Marco},
  journal={Foundations and Trends in Machine Learning},
  volume={11},
  number={5-6},
  pages={355--607},
  year={2019}
}

@inproceedings{rabin2011wasserstein,
  title={Wasserstein Barycenter and Its Application to Texture Mixing},
  author={Rabin, Julien and Peyr{\'e}, Gabriel and Delon, Julie and Bernot, Marc},
  booktitle={International Conference on Scale Space and Variational Methods in Computer Vision},
  pages={435--446},
  year={2011},
  publisher={Springer},
  address={Berlin, Heidelberg, Germany},
  doi={10.1007/978-3-642-24785-9_37}
}

@article{bonneel2015sliced,
  title={Sliced and Radon Wasserstein Barycenters of Measures},
  author={Bonneel, Nicolas and Rabin, Julien and Peyr{\'e}, Gabriel and Pfister, Hanspeter},
  journal={Journal of Mathematical Imaging and Vision},
  volume={51},
  number={1},
  pages={22--45},
  year={2015},
  doi={10.1007/s10851-014-0506-3}
}

@inproceedings{pele2009fast,
  title={Fast and Robust Earth Mover’s Distances},
  author={Pele, Ofir and Werman, Michael},
  booktitle={IEEE 12th International Conference on Computer Vision},
  pages={460--467},
  year={2009},
  organization={IEEE},
}

@inproceedings{nadjahi2019asymptotic,
  title={Asymptotic Guarantees for Learning Generative Models With the Sliced-Wasserstein Distance},
  author={Nadjahi, Kimia and Durmus, Alain and {\c{S}}im{\c{s}}ekli, Umut and Badeau, Roland},
  booktitle={Advances in Neural Information Processing Systems},
  pages={1--12},
  year={2019},
}

@inproceedings{nadjahi2020statistical,
  title={Statistical and Topological Properties of Sliced Probability Divergences},
  author={Nadjahi, Kimia and Durmus, Alain and Chizat, L{\'e}na{\"\i}c and Kolouri, Soheil and Shahrampour, Shahin and Simsekli, Umut},
  booktitle={Advances in Neural Information Processing Systems 33},
  pages={20802--20812},
  year={2020},
}

@inproceedings{kolouri2016sliced,
  title={Sliced Wasserstein Kernels for Probability Distributions},
  author={Kolouri, Soheil and Zou, Yang and Rohde, Gustavo K},
  booktitle={Proceedings of the IEEE Conference on Computer Vision and Pattern Recognition},
  pages={5258--5267},
  year={2016},
  organization={IEEE},
}

@article{Ahmad2022SeizureDetection,
  title={EEG-Based Epileptic Seizure Detection via Machine/Deep Learning Approaches: A Systematic Review},
  author={Ijaz Ahmad and Xin Wang and Mingxing Zhu and Cheng Wang and Yao Pi and Javed Ali Khan and Siyab Khan and Oluwarotimi Williams Samuel and Shixiong Chen and Guanglin Li},
  journal={Computational Intelligence and Neuroscience},
  volume={2022},
  number={},
  pages={6486570},
  year={2022},
}

@article{Cherian2022SeizureDetection,
  title={Theoretical and Methodological Analysis of EEG Based Seizure Detection and Prediction: An Exhaustive Review},
  author={Resmi Cherian and E. Gracemary Kanaga},
  journal={Journal of Neuroscience Methods},
  volume={369},
  number={},
  pages={109483},
  year={2022},
}

@article{Aboalayon2016SleepClassification,
  title={Sleep Stage Classification Using EEG Signal Analysis: A Comprehensive Survey and New Investigation},
  author={Khald Ali I. Aboalayon and Miad Faezipour and Wafaa S. Almuhammadi and Saeid Moslehpour},
  journal={Entropy},
  volume={18},
  number={9},
  pages={272},
  year={2016},
}

@article{Phan2022SleepStaging,
  author  = {Phan, H. and Mikkelsen, K.},
  title   = {Automatic sleep staging of EEG signals: Recent development, challenges, and future directions},
  journal = {Physiological Measurement},
  year    = {2022},
  volume  = {43},
  number  = {4},
  pages   = {04TR01},
}

@article{Altaheri2023MotorImagery,
  author  = {Altaheri, H. and Muhammad, G. and Alsulaiman, M. and Amin, S. U. and Altuwaijri, G. A. and Abdul, W. and Bencherif, M. A. and Faisal, M.},
  title   = {Deep learning techniques for classification of electroencephalogram (EEG) motor imagery (MI) signals: A review},
  journal = {Neural Computing and Applications},
  year    = {2023},
  volume  = {35},
  number  = {20},
  pages   = {14681--14722},
}

@inproceedings{Roy2019ChronoNet,
  title={ChronoNet: A deep recurrent neural network for abnormal EEG identification},
  author={Roy, Subhrajit and Kiral-Kornek, Isabell and Harrer, Stefan},
  booktitle={Conference on Artificial Intelligence in Medicine},
  volume={11526},
  pages={47--56},
  year={2019},
}

@article{Suhaimi2020EmotionRecognition,
  author  = {Nazmi Sofian Suhaimi and James Mountstephens and Jason Teo and others},
  title   = {EEG-based emotion recognition: A state-of-the-art review of current trends and opportunities},
  journal = {Computational Intelligence and Neuroscience},
  year    = {2020},
  volume  = {2020},
  pages   = {9814248},
}

@article{Biesmans2016AuditoryAttention,
  author  = {Wouter Biesmans and Neetha Das and Tom Francart and Alexander Bertrand},
  title   = {Auditory-inspired speech envelope extraction methods for improved EEG-based auditory attention detection in a cocktail party scenario},
  journal = {IEEE Transactions on Neural Systems and Rehabilitation Engineering},
  year    = {2016},
  volume  = {25},
  number  = {5},
  pages   = {402--412},
}

@article{Dadebayev2022EmotionRecognition,
  author  = {Dadebayev, D. and Goh, W. W. and Tan, E. X.},
  title   = {EEG-based emotion recognition: Review of commercial EEG devices and machine learning techniques},
  journal = {Journal of King Saud University-Computer and Information Sciences},
  year    = {2022},
  volume  = {34},
  number  = {7},
  pages   = {4385--4401},
}

@inproceedings{Hine2017RestingStateEEG,
  title={Resting-state EEG: A study on its non-stationarity for biometric applications},
  author={Hine, Gabriel Emile and Maiorana, Emanuele and Campisi, Patrizio},
  booktitle={Proceedings of the 2017 International Conference of the Biometrics Special Interest Group (BIOSIG)},
  pages={1--5},
  year={2017},
  publisher={IEEE},
}

@article{blankertz2007optimizing,
  title={Optimizing spatial filters for robust EEG single-trial analysis},
  author={Blankertz, Benjamin and Tomioka, Ryota and Lemm, Steven and Kawanabe, Motoaki and Muller, Klaus-Robert},
  journal={IEEE Signal processing magazine},
  volume={25},
  number={1},
  pages={41--56},
  year={2007},
  publisher={IEEE}
}

@book{bhatia2009positive,
  title={Positive Definite Matrices},
  author={Bhatia, Rajendra},
  year={2009},
  publisher={Princeton University Press},
  address={Princeton, NJ, USA}
}

@article{mccann2001polar,
  title={Polar factorization of maps on Riemannian manifolds},
  author={McCann, Robert J},
  journal={Geometric \& Functional Analysis GAFA},
  volume={11},
  number={3},
  pages={589--608},
  year={2001},
  publisher={Springer}
}

@article{cui2019spherical,
  title={Spherical optimal transportation},
  author={Cui, Li and Qi, Xin and Wen, Chengfeng and Lei, Na and Li, Xinyuan and Zhang, Min and Gu, Xianfeng},
  journal={Computer-Aided Design},
  volume={115},
  pages={181--193},
  year={2019},
  publisher={Elsevier}
}

@article{ju2022deep,
  title={Deep optimal transport for domain adaptation on SPD manifolds},
  author={Ju, Ce and Guan, Cuntai},
  journal={Artificial Intelligence},
  volume={345},
  year={2025},
  pages={104347},
  publisher={Elsevier},
  address={Amsterdam, The Netherlands},
}

@inproceedings{cuturi2013sinkhorn,
  title={Sinkhorn distances: Lightspeed computation of optimal transport},
  author={Cuturi, Marco},
  booktitle={Advances in Neural Information Processing Systems},
  volume={26},
  pages={2292--2300},
  year={2013},
  publisher={Curran Associates, Inc.},
  address={Lake Tahoe, Nevada, USA}
}

@inproceedings{fatras2019learning,
  title={Learning with minibatch Wasserstein: asymptotic and gradient properties},
  author={Fatras, Kilian and Zine, Younes and Flamary, R{\'e}mi and Gribonval, R{\'e}mi and Courty, Nicolas},
  booktitle={Proceedings of the Twenty Third International Conference on Artificial Intelligence and Statistics},
  volume={108},
  pages={2131--2141},
  year={2020},
  publisher={PMLR},
  address={Online}
}

@phdthesis{bonnotte2013unidimensional,
  title={Unidimensional and evolution methods for optimal transportation},
  author={Bonnotte, Nicolas},
  year={2013},
  school={Paris 11}
}

@article{bayraktar2021strong,
  title={Strong equivalence between metrics of Wasserstein type},
  author={Bayraktar, Erhan and Guo, Gaoyue},
  journal={Electronic Communications in Probability},
  volume={26},
  pages={1--13},
  year={2021},
  publisher={Institute of Mathematical Statistics and Bernoulli Society}
}

@inproceedings{bonet2022spherical,
  title={Spherical Sliced-Wasserstein},
  author={Bonet, Clément and Berg, Paul and Courty, Nicolas and Septier, François and Drumetz, Lucas and Pham, Minh-Tan},
  booktitle={International Conference on Learning Representations},
  numpages={9},
  year={2023},
  publisher={OpenReview.net},
  address={Kigali, Rwanda}
}

@inproceedings{bonet2022hyperbolic,
  title={Hyperbolic Sliced-Wasserstein via Geodesic and Horospherical Projections},
  author={Bonet, Cl{\'e}ment and Chapel, Laetitia and Drumetz, Lucas and Courty, Nicolas},
  booktitle={Proceedings of the 2nd Annual Workshop on Topology, Algebra, and Geometry in Machine Learning at the 40th International Conference on Machine Learning},
  numpages={12},
  year={2023},
  publisher={PMLR},
  address={Honolulu, Hawaii, USA}
}

@inproceedings{bonet2023sliced,
  title={Sliced-Wasserstein on symmetric positive definite matrices for M/EEG signals},
  author={Bonet, Cl{\'e}ment and Mal{\'e}zieux, Beno{\i}t and Rakotomamonjy, Alain and Drumetz, Lucas and Moreau, Thomas and Kowalski, Matthieu and Courty, Nicolas},
  booktitle={International Conference on Machine Learning},
  pages={2777--2805},
  year={2023},
  publisher={PMLR},
  address={Honolulu, Hawaii, USA}
}

@inproceedings{hu2025correlation,
  title={A correlation manifold self-attention network for EEG decoding},
  author={Hu, Chen and Wang, Rui and Song, Xiaoning and Zhou, Tao and Wu, Xiao-Jun and Sebe, Nicu and Chen, Ziheng},
  booktitle={Proceedings of the Thirty-Fourth International Joint Conference on Artificial Intelligence},
  pages={5372--5380},
  year={2025},
}

@techreport{rebonato2011most,
  title={The most general methodology to create a valid correlation matrix for risk management and option pricing purposes},
  author={Rebonato, Riccardo and J{\"a}ckel, Peter},
  institution={Available at SSRN 1969689},
  numpages={12},
  year={2011}
}

@InProceedings{Choi_2021_CVPR,
    author    = {Choi, Sungha and Jung, Sanghun and Yun, Huiwon and Kim, Joanne T. and Kim, Seungryong and Choo, Jaegul},
    title     = {RobustNet: Improving Domain Generalization in Urban-Scene Segmentation via Instance Selective Whitening},
    booktitle = {Proceedings of the IEEE/CVF Conference on Computer Vision and Pattern Recognition},
    month     = {June},
    year      = {2021},
    pages     = {11580-11590}
}

@book{bridson2013metric,
  title={Metric spaces of non-positive curvature},
  author={Bridson, Martin R and Haefliger, Andr{\'e}},
  volume={319},
  year={2013},
  publisher={Springer Science \& Business Media}
}

@inproceedings{lee2023decompose,
  title={Decompose, adjust, compose: Effective normalization by playing with frequency for domain generalization},
  author={Lee, Sangrok and Bae, Jongseong and Kim, Ha Young},
  booktitle={Proceedings of the IEEE/CVF conference on computer vision and pattern recognition},
  pages={11776--11785},
  year={2023}
}

@article{bi2025learning,
  author={Bi, Qi and Yi, Jingjun and Zheng, Hao and Ji, Wei and Huang, Yawen and Li, Yuexiang and Zheng, Yefeng},
  journal={IEEE Transactions on Pattern Analysis and Machine Intelligence}, 
  title={Learning Generalized Medical Image Representation by Decoupled Feature Queries}, 
  year={2025},
  volume={47},
  number={12},
  pages={11252-11269}
}

@inproceedings{bi2024learning,
  title={Learning Generalized Medical Image Segmentation from Decoupled Feature Queries},
  author={Bi, Qi and Yi, Jingjun and Zheng, Hao and Ji, Wei and Huang, Yawen and Li, Yuexiang and Zheng, Yefeng},
  booktitle={Proceedings of the AAAI Conference on Artificial Intelligence},
  year={2024},
  volume={38},
  pages={810--818}
}

@inproceedings{ICLR2025_fc6cd575,
 author = {Li, Shanglin and Kawanabe, Motoaki and Kobler, Reinmar},
 booktitle = {International Conference on Representation Learning},
 title = {SPDIM: Source-Free Unsupervised Conditional and Label Shift Adaptation in EEG},
 volume = {2025},
 year = {2025}
}

@article{ju2025spd,
  title={SPD Learning for Covariance-Based Neuroimaging Analysis: Perspectives, Methods, and Challenges},
  author={Ju, Ce and Kobler, Reinmar J and Collas, Antoine and Kawanabe, Motoaki and Guan, Cuntai and Thirion, Bertrand},
  journal={arXiv:2504.18882},
  year={2025}
}

@article{congedo2017riemannian,
  title={Riemannian geometry for EEG-based brain-computer interfaces; a primer and a review},
  author={Congedo, Marco and Barachant, Alexandre and Bhatia, Rajendra},
  journal={Brain-Computer Interfaces},
  volume={4},
  number={3},
  pages={155--174},
  year={2017},
}

@misc{kingma2017adammethodstochasticoptimization,
      title={Adam: A Method for Stochastic Optimization}, 
      author={Diederik P. Kingma and Jimmy Ba},
      year={2017},
      eprint={Preprint},
      archivePrefix={arXiv},
      primaryClass={cs.LG},
      url={https://arxiv.org/abs/1412.6980}, 
}

@article{bonet2025sliced,
  title={Sliced-Wasserstein Distances and Flows on Cartan-Hadamard Manifolds},
  author={Bonet, Cl{\'e}ment and Drumetz, Lucas and Courty, Nicolas},
  journal={Journal of Machine Learning Research},
  volume={26},
  number={32},
  pages={1--76},
  year={2025}
}

@inproceedings{
li2026heegnet,
title={{HEEGN}et: Hyperbolic Embeddings for {EEG}},
author={Shanglin Li and Chu Shiwen and Okan Ko{\c{c}} and Yi Ding and Qibin Zhao and Motoaki Kawanabe and Ziheng Chen},
booktitle={The Fourteenth International Conference on Learning Representations},
year={2026},
url={https://openreview.net/forum?id=CNDNRjpVIL}
}

@String{Computing = "Computing" }

@String{Computer = "{IEEE} Computer" }

@String{Springer = "Springer-Verlag" }

@ArtifactSoftware{R,
    title = {R: A Language and Environment for Statistical Computing},
    author = {{R Core Team}},
    organization = {R Foundation for Statistical Computing},
    address = {Vienna, Austria},
    year = {2019},
    url = {https://www.R-project.org/},
}

\newpage
\appendix

\section{Related Work}

\paragraph{\textbf{Covariance/Correlation Matrix-Based EEG Decoding}}
Methods that operate on covariance or correlation matrices have demonstrated strong EEG decoding performance by explicitly modeling the underlying geometry of symmetric positive definite (SPD) manifolds, thereby improving robustness to inter-subject variability and session-wise distribution shifts~\citep{congedo2017riemannian,ju2025spd}.
MAtt~\citep{matt} constructs attention mechanisms directly on SPD manifolds, enabling the network to capture spatiotemporal dependencies between EEG channels through geometry-preserving operations, and consistently outperforming conventional Euclidean deep learning baselines.
SPDDSMBN~\citep{kobler2022spd} focuses on unsupervised domain adaptation by learning domain-invariant mappings in the tangent space of the SPD manifold, where an interpretable normalization scheme is introduced to reduce domain-specific covariance shifts.
DGCCA~\citep{ju2024deep} further extends correlation-based modeling by introducing geodesic correlation learning in an SPD latent space, explicitly aligning paired covariance representations across modalities while respecting Riemannian geometry.
SPDIM~\citep{ICLR2025_fc6cd575} addresses the more challenging source-free domain adaptation setting by enforcing SPD-constrained parameterizations to handle both conditional shift and label shift without access to source-domain data.

\paragraph{\textbf{Domain Generalization}}
Domain generalization aims to learn domain-invariant representations from multiple source domains to generalize well to unseen target domains. Recent studies have explored various strategies to align feature distributions or disentangle domain-specific information. IW~\cite{Choi_2021_CVPR} addresses the issue where standard whitening transformations may eliminate domain-invariant content along with style. It proposes an Instance Selective Whitening (ISW) loss to disentangle domain-specific style (encoded in feature covariance) from domain-invariant content, selectively suppressing only the style-sensitive covariance terms to enhance robustness. DAC-SC~\cite{lee2023decompose} approaches domain generalization from a frequency domain perspective, treating amplitude and phase as style and content, respectively. It introduces a Decompose, Adjust, and Compose (DAC) framework with Style-Controlling (SC) normalization to explicitly adjust the degree of style variation while preserving content information. In the context of medical imaging, RIW~\cite{bi2024learning} identifies feature redundancy and misalignment as key barriers to generalization. It proposes a Relaxed Deep Whitening Transformation (RDWT) to minimize channel-wise correlations in a learnable manner, thereby maximizing the expressive power of per-channel representations. Extending this, DRIW~\cite{bi2025learning} introduces a Deep Restricted Isometry Whitening transform. By optimizing the spectral norm to satisfy the restricted isometry property, DRIW theoretically enforces strict orthogonality among feature channels, further reducing redundancy and improving cross-domain feature alignment.

\section{Proofs of the Theorems in the Main Paper}
\label{sec:proofs}
The proofs are organized following the logical dependencies of the results in the main paper,
starting from auxiliary bounds and complexity results, and culminating in the main compactness
and convergence theorems.

\subsection{Proof of the \texorpdfstring{\cref{thm:pem_projection}}{Lg}}
\begin{proof}[Proof of \cref{thm:pem_projection}]
Let $a\in S_{\mathcal{E}}$.
We prove \cref{eq:pem_projection,eq:pem_coordinate} by first characterizing the geodesic
$\mathcal{G}_a$ in closed form, and then deriving the associated slicing coordinate via
geodesic projection.

\noindent\textbf{Part I: Geodesic $\mathcal{G}_a$ in PEMs.}
Consider the geodesic curve
\begin{equation}
\label{eq:proof_pem_gamma_def}
\gamma_a(t)=\phi^{-1}\bigl(\phi(x_0)+t\,a\bigr),\qquad t\in\mathbb{R}.
\end{equation}
Then for any $s,t\in\mathbb{R}$, we have
\begin{equation}
\label{eq:proof_pem_geodesic_chain}
\begin{aligned}
d_{\mathcal{M}}\bigl(\gamma_a(s),\gamma_a(t)\bigr)
&= d_{\mathcal{M}} \left(\phi^{-1}\bigl(\phi(x_0)+s\,a\bigr),\ \phi^{-1}\bigl(\phi(x_0)+t\,a\bigr)\right) \\
&\overset{(1)}{=} \left\|\bigl(\phi(x_0)+s\,a\bigr)-\bigl(\phi(x_0)+t\,a\bigr)\right\|_{\mathcal{E}} \\
&= \|(s-t)a\|_{\mathcal{E}} \\
&= |s-t|\,\|a\|_{\mathcal{E}} \\
&\overset{(2)}{=} |s-t|.
\end{aligned}
\end{equation}
The derivation of \cref{eq:proof_pem_geodesic_chain} follows.

(1) follows from the pullback Euclidean distance
$d_{\mathcal{M}}(x,y)=\|\phi(x)-\phi(y)\|_{\mathcal{E}}$.

(2) follows from $a\in S_{\mathcal{E}}$, i.e., $\|a\|_{\mathcal{E}}=1$.

Thus, $\mathcal{G}_a=\{\gamma_a(t):t\in\mathbb{R}\}$ is a geodesic in $\mathcal{M}$. We denote this geodesic by
\begin{equation}
\label{eq:proof_pem_geodesic_set}
\mathcal{G}_a
=\bigl\{\gamma_a(t)=\phi^{-1}\bigl(\phi(x_0)+t\,a\bigr):t\in\mathbb{R}\bigr\}.
\end{equation}

\smallskip
\noindent\textbf{Part II: Slicing coordinate via geodesic projection.}
Let $x\in\mathcal{M}$. By definition, the geodesic projection of $x$ onto $\mathcal{G}_a$ is defined as
\begin{equation}
\label{eq:proof_pem_proj_def}
P^{\mathcal{G}_a}(x)
=\argmin_{y\in\mathcal{G}_a} d_{\mathcal{M}}(x,y)^2
=\argmin_{t\in\mathbb{R}} d_{\mathcal{M}}\bigl(x,\gamma_a(t)\bigr)^2.
\end{equation}
Let $f(t):=d_{\mathcal{M}}\bigl(x,\gamma_a(t)\bigr)^2$.
Then, we have
\begin{equation}
\label{eq:proof_pem_proj_chain}
\begin{aligned}
f(t)
&= d_{\mathcal{M}} \left(x,\phi^{-1}\bigl(\phi(x_0)+t\,a\bigr)\right)^2 \\
&\overset{(1)}{=} \left\|\phi(x)-\phi  \left(\phi^{-1}\bigl(\phi(x_0)+t\,a\bigr)\right)\right\|_{\mathcal{E}}^2 \\
&\overset{(2)}{=} \left\|\phi(x)-\bigl(\phi(x_0)+t\,a\bigr)\right\|_{\mathcal{E}}^2 \\
&\overset{(3)}{=} \left\langle \phi(x)-\phi(x_0)-t\,a,\ \phi(x)-\phi(x_0)-t\,a \right\rangle_{\mathcal{E}} \\
&\overset{(4)}{=} \left\langle \phi(x)-\phi(x_0),\ \phi(x)-\phi(x_0) \right\rangle_{\mathcal{E}}
- t\left\langle \phi(x)-\phi(x_0),\ a \right\rangle_{\mathcal{E}} \\
&\quad   - t\left\langle a,\ \phi(x)-\phi(x_0) \right\rangle_{\mathcal{E}}
+ t^2\left\langle a,\ a \right\rangle_{\mathcal{E}} \\
&\overset{(5)}{=} \left\|\phi(x)-\phi(x_0)\right\|_{\mathcal{E}}^2
-2t\left\langle a,\ \phi(x)-\phi(x_0)\right\rangle_{\mathcal{E}}
+t^2\langle a,a\rangle_{\mathcal{E}} \\
&\overset{(6)}{=} \left\|\phi(x)-\phi(x_0)\right\|_{\mathcal{E}}^2
-2t\left\langle a,\ \phi(x)-\phi(x_0)\right\rangle_{\mathcal{E}}
+t^2.
\end{aligned}
\end{equation}

The derivation of \cref{eq:proof_pem_proj_chain} follows.

(1) follows from the pullback Euclidean distance
$d_{\mathcal{M}}(x,y)=\|\phi(x)-\phi(y)\|_{\mathcal{E}}$.

(2) follows from $\phi\circ\phi^{-1}$ is the identify map.

(3) follows from $\|z\|_{\mathcal{E}}^2=\langle z,z\rangle_{\mathcal{E}}$.

(4) follows from the bilinearity of $\langle\cdot,\cdot\rangle_{\mathcal{E}}$.

(5) follows from the symmetry of the inner product,
$\langle u,v\rangle_{\mathcal{E}}=\langle v,u\rangle_{\mathcal{E}}$,
and the definition $\|u\|_{\mathcal{E}}^2=\langle u,u\rangle_{\mathcal{E}}$.

(6) follows from $a\in S_{\mathcal{E}}$, i.e., $\langle a,a\rangle_{\mathcal{E}}=1$.

Therefore, $f(t)$ is a quadratic function of $t$.
Its derivative with respect to $t$ is given by
\begin{equation}
\label{eq:proof_pem_proj_deriv}
f'(t)
= -2\left\langle a,\ \phi(x)-\phi(x_0)\right\rangle_{\mathcal{E}} + 2t.
\end{equation}
Since $f'(t)=0$ at the minimizer and $f''(t)=2>0$, the unique minimizer is given by
\begin{equation}
\label{eq:proof_pem_coord_final}
t_a(x)=\left\langle a,\ \phi(x)-\phi(x_0)\right\rangle_{\mathcal{E}},
\end{equation}
which yields the slicing coordinate defined in \cref{eq:pem_coordinate}.

Substituting \cref{eq:proof_pem_coord_final} into the geodesic parametrization
\cref{eq:proof_pem_gamma_def}, we obtain
\begin{equation}
\label{eq:proof_pem_proj_final}
P^{\mathcal{G}_a}(x)
=\gamma_a\bigl(t_a(x)\bigr)
=\phi^{-1}  \Bigl(\phi(x_0)+t_a(x)\,a\Bigr)\in\mathcal{G}_a,
\end{equation}
which yields the projection formula in \cref{eq:pem_projection}.

\end{proof}

There are two classical approaches to defining coordinates on geodesically complete Riemannian manifolds \citep{bonet2022hyperbolic}.
The first relies on geodesic projection, as adopted in the previous section,
while the second is based on Busemann functions \citep{bridson2013metric}.
In the following, we show that on pullback Euclidean manifolds, the Busemann coordinate associated with a geodesic ray admits a closed-form expression and coincides, up to a sign convention, with the slicing coordinate obtained via geodesic projection.

\begin{definition}[Busemann function on PEMs]
Let $(\mathcal{M},g)$ be a geodesically complete pullback Euclidean manifold induced by a diffeomorphism
$\phi:\mathcal{M}\to\mathcal{E}$.
Let $a\in S_{\mathcal{E}}$ and let
\begin{equation}
\gamma_a(t)=\phi^{-1}\bigl(\phi(x_0)+t,a\bigr),\qquad t\ge 0,
\end{equation}
be a geodesic ray on $\mathcal{M}$.
The Busemann function associated with $\gamma_a$ is defined as
\begin{equation}
\label{eq:pem_busemann_def}
B_{\gamma_a}(x)
=\lim_{t\to\infty}\bigl(d_{\mathcal{M}}(x,\gamma_a(t))-t\bigr),
\qquad \forall x\in\mathcal{M}.
\end{equation}
\end{definition}

\begin{proposition}[Busemann coordinates on PEMs]
\label[proposition]{prop:pem_busemann_coords}
Let $a\in S_{\mathcal{E}}$ and let $\gamma_a$ be the geodesic ray defined above.
Then, for any $x\in\mathcal{M}$, the associated Busemann function admits the closed form
\begin{equation}
\label{eq:pem_busemann_coord}
B_{\gamma_a}(x)
=
-\bigl\langle a,\ \phi(x)-\phi(x_0)\bigr\rangle_{\mathcal{E}}.
\end{equation}
\end{proposition}

\begin{proof}[Proof of \cref{prop:pem_busemann_coords}]
By definition, we have
\begin{equation}
B_{\gamma_a}(x)
=\lim_{t\to\infty}\bigl(d_{\mathcal{M}}(x,\gamma_a(t))-t\bigr).
\end{equation}
Using the pullback Euclidean distance
$d_{\mathcal{M}}(x,y)=\|\phi(x)-\phi(y)\|_{\mathcal{E}}$,
we obtain
\begin{equation}
\label{eq:proof_pem_busemann_chain}
\begin{aligned}
d_{\mathcal{M}}(x,\gamma_a(t))-t
&= \left\|\phi(x)-\bigl(\phi(x_0)+t\,a\bigr)\right\|_{\mathcal{E}} - t \\
&= \left\|\phi(x)-\phi(x_0)-t\,a\right\|_{\mathcal{E}} - t .
\end{aligned}
\end{equation}
Let $u=\phi(x)-\phi(x_0)\in\mathcal{E}$.
Since $a\in S_{\mathcal{E}}$, it follows from standard Euclidean geometry that
\begin{equation}
\label{eq:pem_busemann_limit}
\lim_{t\to\infty}\bigl(\|u-t\,a\|_{\mathcal{E}}-t\bigr)
=-\langle u,a\rangle_{\mathcal{E}}.
\end{equation}
Taking the limit in \cref{eq:proof_pem_busemann_chain} yields
\begin{equation}
B_{\gamma_a}(x)
=-\bigl\langle a,\ \phi(x)-\phi(x_0)\bigr\rangle_{\mathcal{E}},
\end{equation}
which establishes \cref{eq:pem_busemann_coord}.
\end{proof}

\begin{corollary}[Equivalence with slicing coordinates]
\label[corollary]{cor:pem_busemann_slice}
Let $t_a(x)$ denote the slicing coordinate obtained via geodesic projection.
Then
\begin{equation}
B_{\gamma_a}(x) = - t_a(x),
\end{equation}
and the Busemann projection onto $\mathcal{G}_a$ coincides with the geodesic projection.
\end{corollary}

\begin{proof}[Proof of \cref{cor:pem_busemann_slice}]
As shown in \cref{prop:pem_busemann_coords}, the Busemann coordinate of $x$ is
$B_{\gamma_a}(x)=-\langle a,\phi(x)-\phi(x_0)\rangle_{\mathcal{E}}$.
On the other hand, the slicing coordinate obtained by geodesic projection satisfies
$t_a(x)=\langle a,\phi(x)-\phi(x_0)\rangle_{\mathcal{E}}$.
Therefore, $B_{\gamma_a}(x)=-t_a(x)$.

The Busemann projection is defined as the unique point $\gamma_a(t)$ on the geodesic ray
sharing the same Busemann coordinate, i.e.,
$B_{\gamma_a}(x)=B_{\gamma_a}(\gamma_a(t))$.
Since $B_{\gamma_a}(\gamma_a(t))=-t$, we obtain $t=t_a(x)$,
which yields
\begin{equation}
P^{\mathcal{G}_a}(x)
=\gamma_a\bigl(t_a(x)\bigr)
=\phi^{-1}\Bigl(\phi(x_0)+t_a(x)\,a\Bigr).
\end{equation}
Hence, the Busemann projection coincides with the geodesic projection.
\end{proof}

\subsection{Proof of the \texorpdfstring{\cref{lem:translation_projection}}{Lg}}
\label{sec:proof:lem_translation_projection}

\begin{proof}[Proof of \cref{lem:translation_projection}]
By the definition of the slicing coordinate in \cref{eq:pem_coordinate}, we have
\begin{equation}
\label{eq:proof_ta_def}
t_a(x)
=\bigl\langle a,\ \phi(x)-\phi(x_0)\bigr\rangle_{\mathcal{E}}.
\end{equation}
Let $\pi_a:\mathcal{E}\to\mathbb{R}$ be defined by $\pi_a(z)=\langle a,z\rangle_{\mathcal{E}}$.
Then \cref{eq:proof_ta_def} can be equivalently written as
\begin{equation}
\label{eq:proof_ta_shift}
t_a(x)
=\pi_a(\phi(x))-\pi_a(\phi(x_0)),
\end{equation}
which establishes \cref{eq:ta_shifted}.

Let $\tilde{\mu}=\phi_\#\mu$.
Since
\(
t_a=\tau_{-\pi_a(\phi(x_0))}\circ \pi_a\circ\phi
\),
where $\tau_c:\mathbb{R}\to\mathbb{R}$ denotes the translation $\tau_c(s)=s+c$,
it follows that
\begin{equation}
t_a{}_\#\mu
=(\tau_{-\pi_a(\phi(x_0))})_\#\bigl(\pi_a{}_\#\tilde{\mu}\bigr).
\end{equation}
Therefore, $t_a{}_\#\mu$ is a translation of $\pi_a{}_\#\tilde{\mu}$.
It can be readily verified that the same conclusion holds for $\nu$.
\end{proof}

\subsection{Proof of the \texorpdfstring{\cref{thm:pem_sw_euclidean}}{Lg}}
\label{sec:proof:thm_pem_sw_euclidean}

\begin{proof}[Proof of $\mathrm{PEMSW}_p(\mu,\nu)=\mathrm{SW}_p(\tilde{\mu},\tilde{\nu})$]
For each $a\in S_{\mathcal{E}}$, we consider the linear functional
$z\mapsto\langle a,z\rangle_{\mathcal{E}}$ on $\mathcal{E}$.
By \cref{lem:translation_projection}, this functional induces the slicing coordinate
\begin{equation}
\label{eq:proof_pem_ta_pi}
t_a(x)=\langle a,\phi(x)\rangle_{\mathcal{E}}-\langle a,\phi(x_0)\rangle_{\mathcal{E}},
\qquad \forall\,x\in\mathcal{M}.
\end{equation}
Let $\tilde{\mu}=\phi_\#\mu$ and $\tilde{\nu}=\phi_\#\nu$.
Using the standard pushforward representation of the $p$-Wasserstein distance
(e.g., \citep[Lemma~6]{paty2019subspace}),
we obtain
\begin{equation}
\label{eq:proof_pem_equiv_chain}
\begin{aligned}
W_p^p\bigl(t_a{}_\#\mu,\ t_a{}_\#\nu\bigr)
&= \inf_{\gamma\in\Pi(\mu,\nu)}
\int_{\mathcal{M}\times\mathcal{M}}
\bigl|t_a(x)-t_a(y)\bigr|^p\,\mathrm{d}\gamma(x,y) \\
&\overset{(1)}{=} \inf_{\gamma\in\Pi(\mu,\nu)}
\int_{\mathcal{M}\times\mathcal{M}}
\Bigl|\langle a,\phi(x)\rangle_{\mathcal{E}}
      -\langle a,\phi(y)\rangle_{\mathcal{E}}\Bigr|^p\,
\mathrm{d}\gamma(x,y) \\
&= W_p^p\bigl((\langle a,\cdot\rangle_{\mathcal{E}})_\#\tilde{\mu},\
              (\langle a,\cdot\rangle_{\mathcal{E}})_\#\tilde{\nu}\bigr),
\end{aligned}
\end{equation}
where $(\phi\times\phi)_\#\gamma\in\Pi(\tilde{\mu},\tilde{\nu})$.

The derivation of \cref{eq:proof_pem_equiv_chain} follows.

(1) follows from \cref{eq:proof_pem_ta_pi}, since the constant terms
$\langle a,\phi(x_0)\rangle_{\mathcal{E}}$ cancel out in the difference
$t_a(x)-t_a(y)$.

By integrating both sides of \cref{eq:proof_pem_equiv_chain} with respect to $a\in S_{\mathcal{E}}$,
we obtain
\begin{equation}
\mathrm{PEMSW}_p^p(\mu,\nu)=\mathrm{SW}_p^p(\tilde{\mu},\tilde{\nu}).
\end{equation}

\end{proof}

\subsection{Proof of the \texorpdfstring{\cref{prop:pem_x0_invariance}}{Lg}}
\label{sec:proof:prop_pem_x0_invariance}

\begin{proof}[Proof of \cref{prop:pem_x0_invariance}]
For any $a\in S_{\mathcal{E}}$, the slicing coordinate associated with $x_i$ is given by
\begin{equation}
t^{(x_i)}_a(x)=\langle a,\phi(x)-\phi(x_i)\rangle_{\mathcal{E}},\quad i\in\{0,1\}.
\end{equation}
Therefore, for all $x\in\mathcal{M}$,
\begin{equation}
t^{(x_1)}_a(x)
=
t^{(x_0)}_a(x)
+\langle a,\phi(x_0)-\phi(x_1)\rangle_{\mathcal{E}}.
\end{equation}
Hence, there exists a constant $c_a\in\mathbb{R}$ such that
\begin{equation}
t^{(x_1)}_a=\tau_{c_a}\circ t^{(x_0)}_a,
\end{equation}
where $\tau_{c_a}(s)=s+c_a$ denotes a translation on $\mathbb{R}$.
Consequently,
\begin{equation}
\begin{aligned}
t^{(x_1)}_a{}_\#\mu
&=(\tau_{c_a})_\#\bigl(t^{(x_0)}_a{}_\#\mu\bigr), \\
t^{(x_1)}_a{}_\#\nu
&=(\tau_{c_a})_\#\bigl(t^{(x_0)}_a{}_\#\nu\bigr).
\end{aligned}
\end{equation}
Since the $p$-Wasserstein distance on $\mathbb{R}$ is invariant under translations, it follows that
\begin{equation}
W_p\Bigl(t^{(x_0)}_a{}_\#\mu,\ t^{(x_0)}_a{}_\#\nu\Bigr)
=
W_p\Bigl(t^{(x_1)}_a{}_\#\mu,\ t^{(x_1)}_a{}_\#\nu\Bigr),
\quad \forall\,a\in S_{\mathcal{E}}.
\end{equation}
This establishes the reference-point invariance of $\mathrm{PEMSW}_p$.
\end{proof}

\subsection{Proof of the \texorpdfstring{\cref{prop:pem_mean_lower}}{Lg}}
\label{sec:proof:prop_pem_mean_lower}

This subsection derives a lower bound on $\mathrm{PEMSW}_p$ in terms of the difference between the Euclidean means of the embedded measures.
The proof first bounds each one-dimensional Wasserstein term from below by the difference of one-dimensional means and then averages the result over $S_{\mathcal{E}}$.

\begin{proof}[Proof of \cref{prop:pem_mean_lower}]
Let $p\ge 1$ and let $\mu,\nu\in\mathcal{P}_p(\mathcal{M})$.
The proof again uses $\tilde{\mu}=\phi_\#\mu$ and $\tilde{\nu}=\phi_\#\nu$.
The proof also fixes $a\in S_{\mathcal{E}}$ and applies a Jensen-type bound to any coupling $\gamma\in\Pi(\mu,\nu)$:
\begin{equation}
\label{eq:proof_pem_mean_jensen}
\int_{\mathcal{M}\times\mathcal{M}} |t_a(x)-t_a(y)|^p\,\mathrm{d}\gamma(x,y)
\ge
\left|
\int_{\mathcal{M}\times\mathcal{M}} \bigl(t_a(x)-t_a(y)\bigr)\,\mathrm{d}\gamma(x,y)
\right|^p.
\end{equation}
The infimum over $\gamma\in\Pi(\mu,\nu)$ turns \cref{eq:proof_pem_mean_jensen} into a pointwise lower bound for $W_p$:
\begin{equation}
\label{eq:proof_pem_mean_wp_lower}
W_p\bigl(t_a{}_\#\mu,\ t_a{}_\#\nu\bigr)^p
\ge
\left|
\int_{\mathcal{M}} t_a(x)\,\mathrm{d}\mu(x)
-
\int_{\mathcal{M}} t_a(y)\,\mathrm{d}\nu(y)
\right|^p.
\end{equation}

The slicing coordinate formula in \cref{eq:pem_coordinate} expresses $t_a$ as
$t_a(x)=\langle a,\phi(x)-\phi(x_0)\rangle_{\mathcal{E}}$.
This representation implies that the reference term cancels in the difference of expectations, and the proof obtains
\begin{equation}
\label{eq:proof_pem_mean_means}
\begin{aligned}
\int_{\mathcal{M}} t_a(x)\,\mathrm{d}\mu(x)
&-
\int_{\mathcal{M}} t_a(y)\,\mathrm{d}\nu(y) \\
&=
\left\langle a,\ \int_{\mathcal{M}}\phi(x)\,\mathrm{d}\mu(x)-\int_{\mathcal{M}}\phi(y)\,\mathrm{d}\nu(y)\right\rangle_{\mathcal{E}} \\
&=
\langle a,\ m_{\tilde{\mu}}-m_{\tilde{\nu}}\rangle_{\mathcal{E}},
\end{aligned}
\end{equation}
where $m_{\tilde{\mu}}=\int_{\mathcal{E}} z\,\mathrm{d}\tilde{\mu}(z)$ and
$m_{\tilde{\nu}}=\int_{\mathcal{E}} z\,\mathrm{d}\tilde{\nu}(z)$.
The combination of \cref{eq:proof_pem_mean_wp_lower,eq:proof_pem_mean_means} yields
\begin{equation}
\label{eq:proof_pem_mean_pointwise}
W_p\bigl(t_a{}_\#\mu,\ t_a{}_\#\nu\bigr)^p
\ge
\bigl|\langle a,\ m_{\tilde{\mu}}-m_{\tilde{\nu}}\rangle_{\mathcal{E}}\bigr|^p.
\end{equation}

The definition of $\mathrm{PEMSW}_p$ integrates $W_p^p(t_a{}_\#\mu,t_a{}_\#\nu)$ over $S_{\mathcal{E}}$.
This definition and \cref{eq:proof_pem_mean_pointwise} yield
\begin{equation}
\label{eq:proof_pem_mean_integrate}
\mathrm{PEMSW}_p(\mu,\nu)^p
\ge
\int_{S_{\mathcal{E}}}
\bigl|\langle a,\ m_{\tilde{\mu}}-m_{\tilde{\nu}}\rangle_{\mathcal{E}}\bigr|^p\,\mathrm{d}\lambda(a).
\end{equation}
Lemma~\ref{lem:sphere_kappa} evaluates the remaining spherical integral and gives
\begin{equation}
\int_{S_{\mathcal{E}}}
\bigl|\langle a,\ m_{\tilde{\mu}}-m_{\tilde{\nu}}\rangle_{\mathcal{E}}\bigr|^p\,\mathrm{d}\lambda(a)
=
\kappa_{D,p}\,\|m_{\tilde{\mu}}-m_{\tilde{\nu}}\|_{\mathcal{E}}^p.
\end{equation}
The last identity and \cref{eq:proof_pem_mean_integrate} establish \cref{eq:pem_mean_lower}.
\end{proof}

\subsection{Proof of the \texorpdfstring{\cref{prop:pem_maxsw}}{Lg}}
\label{sec:proof:prop_pem_maxsw}

This subsection proves the two-sided comparison between $\mathrm{PEMSW}_p$, the max-sliced variant, and $W_p^{\mathcal{M}}$.
The first inequality compares an $L^p$-average with a supremum, while the second inequality uses a $1$-Lipschitz control of the slicing maps.

\begin{proof}[Proof of \cref{prop:pem_maxsw}]
Let $p\ge 1$ and let $\mu,\nu\in\mathcal{P}_p(\mathcal{M})$.

\noindent\textbf{(i) The proof shows that $\mathrm{PEMSW}_p(\mu,\nu)\le \mathrm{PEM\text{-}MaxSW}_p(\mu,\nu)$.}
The definition of the supremum implies that every direction $a\in S_{\mathcal{E}}$ satisfies
\begin{equation}
W_p\bigl(t_a{}_\#\mu,\ t_a{}_\#\nu\bigr)
\le
\sup_{b\in S_{\mathcal{E}}} W_p\bigl(t_b{}_\#\mu,\ t_b{}_\#\nu\bigr).
\end{equation}
The proof raises both sides to the power $p$ and integrates with respect to $\lambda$:
\begin{equation}
\int_{S_{\mathcal{E}}}
W_p\bigl(t_a{}_\#\mu,\ t_a{}_\#\nu\bigr)^p\,\mathrm{d}\lambda(a)
\le
\int_{S_{\mathcal{E}}}
\left(\sup_{b\in S_{\mathcal{E}}} W_p\bigl(t_b{}_\#\mu,\ t_b{}_\#\nu\bigr)\right)^p
\mathrm{d}\lambda(a).
\end{equation}
Because $\lambda$ is a probability measure, the integral on the right-hand side equals the constant
$\left(\sup_{b\in S_{\mathcal{E}}} W_p(t_b{}_\#\mu,t_b{}_\#\nu)\right)^p$.
The definition of $\mathrm{PEMSW}_p$ completes the first inequality:
\begin{equation}
\mathrm{PEMSW}_p(\mu,\nu)
\le
\sup_{a\in S_{\mathcal{E}}} W_p\bigl(t_a{}_\#\mu,\ t_a{}_\#\nu\bigr)
=
\mathrm{PEM\text{-}MaxSW}_p(\mu,\nu).
\end{equation}

\noindent\textbf{(ii) The proof shows that $\mathrm{PEM\text{-}MaxSW}_p(\mu,\nu)\le W_p^{\mathcal{M}}(\mu,\nu)$.}
The proof fixes $a\in S_{\mathcal{E}}$ and chooses an optimal coupling $\gamma^\star\in\Pi(\mu,\nu)$ for $W_p^{\mathcal{M}}(\mu,\nu)$.
The definition of $W_p$ gives
\begin{equation}
W_p\bigl(t_a{}_\#\mu,\ t_a{}_\#\nu\bigr)^p
\le
\int_{\mathcal{M}\times\mathcal{M}} |t_a(x)-t_a(y)|^p\,\mathrm{d}\gamma^\star(x,y).
\end{equation}
The proof rewrites the integrand by using \cref{eq:pem_coordinate} and cancellation of the reference term:
\begin{equation}
\label{eq:proof_pem_maxsw_rewrite}
|t_a(x)-t_a(y)|
=
\bigl|\langle a,\phi(x)-\phi(y)\rangle_{\mathcal{E}}\bigr|.
\end{equation}
Cauchy--Schwarz and $\|a\|_{\mathcal{E}}=1$ yield the pointwise estimate
\begin{equation}
\bigl|\langle a,\phi(x)-\phi(y)\rangle_{\mathcal{E}}\bigr|
\le
\|\phi(x)-\phi(y)\|_{\mathcal{E}}.
\end{equation}
The pullback Euclidean identity $d_{\mathcal{M}}(x,y)=\|\phi(x)-\phi(y)\|_{\mathcal{E}}$ therefore yields
\begin{equation}
W_p\bigl(t_a{}_\#\mu,\ t_a{}_\#\nu\bigr)^p
\le
\int_{\mathcal{M}\times\mathcal{M}} d_{\mathcal{M}}(x,y)^p\,\mathrm{d}\gamma^\star(x,y)
=
W_p^{\mathcal{M}}(\mu,\nu)^p.
\end{equation}
The proof finishes by taking the supremum over $a\in S_{\mathcal{E}}$ and combining the two steps.
\end{proof}

\subsection{Proof of the \texorpdfstring{\cref{cor:olm_geodesic_projection}}{Lg}}
\label{sec:proof:cor_olm_geodesic_projection}

\begin{proof}[Proof of \cref{cor:olm_geodesic_projection}]
The slicing coordinate in \cref{eq:pem_coordinate} states that, for any $x\in\mathcal{M}$ and any unit direction $a\in S_{\mathcal{E}}$,
\begin{equation}
\label{eq:cor_olm_from_pem}
t_a(x)=\langle a,\phi(x)-\phi(x_0)\rangle_{\mathcal{E}}.
\end{equation}
In the present setting, the embedding space $\mathcal{E}$ is identified with the OLM tangent-type space equipped with the Frobenius inner product, so $\langle\cdot,\cdot\rangle_{\mathcal{E}}=\langle\cdot,\cdot\rangle_F$ and $a$ is written as $A\in S_{\mathrm{Hol}}$.
Substituting $\phi=\offlog$, $x=C$, and $x_0=I_n$ into \cref{eq:cor_olm_from_pem} gives
\begin{equation}
t^{o}_{A}(C)=\langle A,\offlog(C)-\offlog(I_n)\rangle_F.
\end{equation}
The construction of $\offlog$ ensures $\offlog(I_n)=0$, and therefore the coordinate reduces to
\begin{equation}
t^{o}_{A}(C)=\langle A,\offlog(C)\rangle_F,
\end{equation}
which is exactly \cref{eq:olm_pem_coordinate_cor}.
\end{proof}

\subsection{Proof of the \texorpdfstring{\cref{cor:lsm_geodesic_projection}}{Lg}}
\label{sec:proof:cor_lsm_geodesic_projection}

\begin{proof}[Proof of \cref{cor:lsm_geodesic_projection}]
The same specialization of \cref{eq:pem_coordinate} yields
\begin{equation}
\label{eq:cor_lsm_from_pem}
t_A^{\star}(C)=\langle A,\lslog(C)-\lslog(I_n)\rangle_F,
\end{equation}
because the LSM embedding space also uses the Frobenius inner product and the unit direction is written as $A\in S_{\mathrm{Row}}$.
The definition of $\lslog$ satisfies $\lslog(I_n)=0$, so \cref{eq:cor_lsm_from_pem} simplifies to
\begin{equation}
t_A^{\star}(C)=\langle A,\lslog(C)\rangle_F,
\end{equation}
which matches \cref{eq:lsm_pem_coordinate_cor}.
\end{proof}

\end{document}